\documentclass[12pt]{article}

\usepackage{graphicx}

\newcommand{\blackslug}{\hbox{\hskip 1pt
        \vrule width 4pt height 8pt depth 1.5pt\hskip 1pt}}
\newcommand{\myQED}{\hfill \blackslug}
\newcommand{\la}{\langle}
\newcommand{\ra}{\rangle}

\newenvironment{proof}
    {\pagebreak[1]{\narrower\noindent {\bf Proof:\nopagebreak}}}%
    {\myQED}

    {\myQED}

\newtheorem{lemma}{Lemma}

\begin{document}

\begin{center}
{\large \bf
Exploring Viable Algorithmic Options for \\
Learning from Demonstration (LfD): \\
A Parameterized Complexity Approach}

\vspace*{0.2in}

Todd Wareham \\
Department of Computer Science \\
Memorial University of Newfoundland \\
St.\ John's, NL Canada \\
(Email: {\tt harold@mun.ca}) \\

\vspace*{0.1in}

\today
\end{center}

\begin{quote}
{\bf Abstract}:
The key to reconciling the polynomial-time intractability of many machine learning 
tasks in the worst case with the surprising solvability of these tasks by heuristic 
algorithms in practice seems to be exploiting restrictions on real-world data sets.
One approach to investigating such restrictions is to analyze why
heuristics perform well under restrictions. A complementary 
approach would be to systematically determine under which sets of restrictions 
efficient and reliable machine learning algorithms do and do not exist. In this
paper, we show how such a systematic exploration of algorithmic options can
be done using parameterized complexity analysis, As an illustrative example, we give
the first parameterized complexity analysis of batch and incremental policy 
inference under Learning from Demonstration (LfD). 
Relative to a basic model of LfD, we show that none of our problems can be solved
efficiently either in general or relative to a number of (often simultaneous) 
restrictions on environments, demonstrations, and policies. We also give the first
known restrictions under which efficient solvability is possible and discuss the
implications of our solvability and unsolvability results for both our basic model
of LfD and more complex models of LfD used in practice.
\end{quote}

\section{Introduction}

\label{SectIntro}

In an ideal world world, one wants algorithms for machine learning tasks that
are both efficient and reliable, in the sense that the algorithms quickly compute the
correct outputs for all possible inputs of interest.
An apparent paradox of machine learning research is that while
many machine learning tasks are $NP$-hard in the worst case and hence cannot be
solved both efficiently and reliably in general,
these tasks are solvable amazingly well in
practice using heuristic algorithms \cite[page 1]{Moi19}. The resolution of this
paradox is that machine learning tasks encountered in practice are
characterized by restrictions on input data sets that allow heuristics to
perform far better than suggested by worst case analyses \cite[page 2]{Moi19}. 
One approach to exploiting these restrictions pioneered by Moitra and others
is to rigorously analyze existing heuristics operating relative to such restrictions
to explain the good performance of those heuristics in practice. This in turn
often suggests fundamentally new ways of solving machine learning tasks.

A complementary approach would be to characterize those combinations of restrictions for
which efficient and reliable algorithms for a given machine learning task do and do not
exist. This can be done using techniques from the theory of parameterized computational
complexity \cite{DF99,DF13,FG06}. If these techniques are applied systematically to all
possible subsets of a given set of plausible restrictions for the task of interest, the
resulting overview of algorithmic options for that task relative to those restrictions 
would be most useful in both deriving the best possible solutions for given task 
instances (by allowing lookup of the most appropriate algorithms relative to 
restrictions characterizing those instances) and 
productively directing research on new efficient algorithms for that task (by 
highlighting those restrictions under which such algorithms can exist).

In this paper, we will show how parameterized complexity analysis can be used to
systematically explore algorithmic options for fast and reliable machine learning. We 
will first give an overview of parameterized complexity analysis (Section 
\ref{SectPCA}). Such analyses are then demonstrated via the first parameterized 
complexity analysis of a classic machine learning task, learning from demonstration 
(LfD) \cite{AC+09,BC+08} (Section \ref{SectCS}). Our analysis is done relative a basic
model of LfD (formalized in Section \ref{SectCSForm}) based on that given in 
\cite{AC+09}, in which discrete feature-based positive and negative demonstrations are
used to infer time-independent policies specified as single-state transducers. We show 
that neither batch nor incremental LfD can be done efficiently in general (Section 
\ref{SectCSCIRes}) or under many (but not all) subsets of a given set of plausible 
restrictions on environments, demonstrations, and policies.(Section \ref{SectCSPrmRes}).
To illustrate how parameterized complexity analyses are performed, proofs of selected 
results are given in the main text; all remaining proofs are given in an appendix.
Finally, after discussing the implications of our results for both our basic model of 
LfD and LfD in practice (Section \ref{SectDisc}), we give our conclusions and some 
promising directions for future research (Section \ref{SectConc}). 

\section{Parameterized Complexity Analysis}

\label{SectPCA}

In this section, we first review of how classical types of 
computational complexity analysis such as the theory of $NP$-completeness 
\cite{GJ79} are used to show that problems are not efficiently solvable in 
general. We then give an overview of the analogous mechanisms from 
parameterized complexity theory \cite{DF99,DF13,FG06} used to show that problems 
are not efficiently solvable under restrictions. Finally, we show how parameterized
complexity analysis can be used to systematically explore efficient and reliable  
algorithmic options 
for solving a problem under restrictions and give several useful rules of thumb 
for minimizing the effort involved in carrying out such analyses.

Both classical and parameterized complexity analyses are based on
the same notion of a computational problem expressed as 
a relation between inputs and their associated outputs. Given an input
instance, a search problem asks for the associated output itself, e.g.,

\vspace*{0.1in}

\noindent
{\sc Dominating set} (search version) \\
{\em Input}: An undirected graph $G = (V, E)$. \\
{\em Output}: A minimum dominating set of $G$,
              i.e., a subset $V' \subseteq V$ of the smallest possible size
              such that for all $v \in V$, either $v \in V'$ or there is at least one
              $v' \in V'$ such that $(v, v') \in E$.

\vspace*{0.1in}

\noindent
In classical complexity analysis, an algorithm for a problem is efficient if that
algorithm runs in polynomial time---that is, the algorithm's running time is 
always upper-bounded by $n^c$ where $n$ is the size of the input and $c$ is a constant.
A problem which has a polynomial-time algorithm is \emph{polynomial-time tractable}.
Polynomial-time algorithms are preferable because their runtimes grow much more slowly 
than algorithms with non-polynomial runtimes, e.g., $2^n$, as input size increases 
and hence allow the solution of much larger inputs in practical amounts of time. 

One shows that a problem $\Pi'$ is not polynomial-time tractable by giving a
reduction from a problem $\Pi$ that is either not polynomial-time tractable or
not polynomial-time tractable unless a widely-believed conjecture such as
$P \neq NP$ \cite{For09} is false.  
A polynomial-time reduction from $\Pi$ to $\Pi'$ 
\cite{GJ79} is essentially a polynomial-time algorithm for transforming instances
of $\Pi$ into instances of $\Pi'$ such that any polynomial-time algorithm for $\Pi'$ 
can be used in conjunction with this instance transformation algorithm to create
a polynomial-time algorithm for $\Pi$. Polynomial-time intractable problems are
isolated from appropriate classes of problems using the notions of hardness and
completeness. Relative to a class of $C$ of problems,
if every problem in $C$ reduces to $\Pi$ then $\Pi$ is said to be \emph{$C$-hard};
if $\Pi$ is also in $C$ then $\Pi$ is \emph{$C$-complete}. For technical reasons, 
these reductions are typically done between decision versions of problems for which
the output is the answer to a yes/no question, e.g.,

\vspace*{0.1in}

\noindent
{\sc Dominating set} (decision version) \\
{\em Input}: An undirected graph $G = (V, E)$ and a positive integer $k$. \\
{\em Question}: Does $G$ contain a dominating set of size $k$?

\vspace*{0.1in}

\noindent
Let such a decision version of a problem $\Pi$ be denoted by $\Pi_D$. This focus on
decision problems does not cause difficulties in practice because if a decision version
$\Pi_D$ of a search problem $\Pi$ is defined such that any algorithm for $\Pi$ can be 
used to solve $\Pi_D$, then the polynomial-time intractability of $\Pi_D$ also implies 
the polynomial-time intractability of $\Pi$. Such is the case for the decision and 
search versions of {\sc Dominating set} defined above. In the case of $NP$-hard decision
problems, this intractability holds unless $P = NP$. This is encoded in the following 
useful lemma.

\begin{lemma}
If {\bf X}$_D$ is $NP$-hard then {\bf X} is not solvable in polynomial time
unless $P = NP$.
\label{LemProp1}
\end{lemma}

Parameterized problems differ from the classical search and decision problems defined 
above in that each parameterized problem has an associated set of one or more 
parameters, where a \emph{parameter} of a problem is an aspect of that problem's 
input or output. Example input and output parameters of {\sc Dominating set}$_D$ are
the maximum degree $d$ of any vertex in the given graph $G$ and the size $k$ of
the requested dominating set. Given a set $K$ of parameters relative to a
problem $\Pi$, let $\la K \ra$-$\Pi$ denote $\Pi$ parameterized relative to $K$.
For example, two parameterized problems associated with {\sc Dominating set}$_D$ 
are $\la k \ra$-{\sc Dominating set}$_D$ and 
$\la d, k \ra$-{\sc Dominating set}$_D$. 

A restriction on a problem is phrased in 
terms of restrictions on the value the corresponding parameter, and algorithm 
efficiency under restrictions is phrased in terms of fixed-parameter tractability.
A problem $\Pi$ is \emph{fixed-parameter (fp-)tractable 
relative a set of parameters $K = \{k_1, k_2, \ldots, k_m\}$} \cite{DF99,DF13},
i.e., $\langle K \rangle$-$\Pi$ is fp-tractable, if there is an algorithm 
for $\Pi$ whose running time is upper-bounded by $f(K)n^c$ for some function
$f()$ where $n$ is the problem input size and $c$ is a constant. 
Fixed-parameter tractability generalizes polynomial-time solvability by 
allowing problems to be effectively solvable in polynomial time when the 
values of the parameters in $K$ are small, e.g., $k_1, k_2 \leq 4$, and $f()$ 
is well-behaved, e.g., $1.2^{k_1 + k_2}$, such that the value of $f(K)$ is a
small constant. Hence, if a polynomial-time intractable problem 
$\Pi$ is fp-tractable relative to a well-behaved $f()$ for a parameter-set $K$
then $\Pi$ can be efficiently solved even for large inputs in which the values of the
parameters in $K$ are small.

One shows that a parameterized problem $\Pi'$ is not fixed-parameter tractable by giving
a parameterized reduction from a parameterized problem $\Pi$ that is either not 
fixed-parameter tractable or not fixed-parameter tractable unless a widely-believed 
conjecture such as $FPT \neq W[1]$ \cite{DF99,DF13} is false. A parameterized reduction from
$\langle K \rangle$-$\Pi$ to $\langle K'\rangle$-$\Pi'$ \cite{DF99} allows the instance
transformation algorithm to run in fp-time relative to $K$ and
requires for each $k' \in K'$ that there is a function $g_{k'}()$ such that
$k' = g_{k'}(K)$. Such an instance transformation algorithm can be used in conjunction
with any fixed-parameter algorithm for $\langle K'\rangle$-$\Pi'$ to create a 
fixed-parameter algorithm for
$\langle K\rangle$-$\Pi$. Hardness and completeness for parameterized reductions is
typically done relative to classes in the $W$-hierarchy $= \{W[1], W[2], \ldots, 
W[P], \ldots\}$ \cite{DF99,DF13}. Once again, for technical reasons, reductions 
are typically done between decision versions of parameterized problems, and as
any algorithm for a search version of a parameterized problem can solve the
appropriately-defined decision version, we have the following parameterized 
analogue of Lemma \ref{LemProp1}.

\begin{lemma}
Given a parameter-set $K$ for problem {\bf X}, if $\la K \ra$-{\bf X}$_D$
is $W[1]$-hard then $\la K \ra$-{\bf X} is not fp-tractable unless $FPT = W[1]$.
\label{LemProp2a}
\end{lemma}

\noindent
In certain situations, one can get a more powerful result.

\begin{lemma}
\cite[Lemma 2.1.35]{War99}
Given a parameter-set $K$ for problem {\bf X}, if {\bf X}$_D$
is $NP$-hard when the value of every parameter $k \in K$ is fixed to a
constant value, then $\la K \ra$-{\bf X} is not fp-tractable unless $P = NP$.
\label{LemProp2}
\end{lemma}

We can now finally talk about how the results of a parameterized complexity 
analysis for a problem can be used to derive an intractability map \cite{War99}, 
which corresponds to the desired systematic overview of algorithmic options for 
solving that problem described in Section \ref{SectIntro}. Given a set $P$ of
parameters of a problem $\Pi$, an intractability map describes the parameterized
complexity status of $\Pi$ relative to each of the $2^{|P|} - 1$ non-empty
subsets of $P$. The choice of $P$ depends on how one wants to use the map. If one
wishes to use the map as a probe to examine the effects of various parameters
on the computational complexity of our problem of interest (as we do in our 
parameterized complexity analysis of learning from demonstration in Section 
\ref{SectCS}), $P$ should consist of parameters (which need not all be of small 
value in practice) characterizing all aspects of the input and output of 
that problem.  If on the other hand one wishes to use the map as a guide to
either developing algorithms or selecting the most appropriate algorithms for 
input instances of that problem that are encountered in practice, $P$ should
consist purely of aspects of the problem that are known to be small in at least
some of these instances. 

It is important to note that the initial algorithms used to construct an 
intractability map need not have practical runtimes---at this stage in analysis, one 
need only establish the fact and not the best possible degree of fp-tractability. The 
best possible fixed-parameter algorithms are developed subsequently as
needed. There are a number of techniques for deriving fixed-parameter
algorithms \cite{CF+15,FL+19,Nie06}, and it has been observed multiple times 
within the parameterized complexity community that once fp-tractability is 
established, these techniques are applied by different groups of researchers in 
``FPT Races'' to produce increasingly (and, on occasion, spectacularly) more 
efficient algorithms \cite{KN12,Ste12}. 

\begin{figure}[t]
\begin{center}
\begin{tabular}{| c || c | c || c | c | c || c | c | c | c |}
\cline{1-4} \cline{6-10}
    & R1 & R2 & {\bf R3} & &
    & -- & $C$ & $D$ & $C,D$ \\
\cline{1-4} \cline{6-10}
\cline{1-4} \cline{6-10}
$A$ & @ & -- & {\bf @} & &
-- & NPh & X & X & X \\
\cline{1-4} \cline{6-10}
$B$ & -- & 4 & {\bf @}& {\LARGE $\Rightarrow$} &
$A$ & X & X & X & X$^{R1}$ \\
\cline{1-4} \cline{6-10}
$C$ & @ & @ & -- & &
$B$ & X & X$^{R2}$ & ??? & ??? \\
\cline{1-4} \cline{6-10}
$D$ & 2 & -- & -- & &
$A,B$ & $\surd^{R3}$ & $\surd$ & $\surd$ & $\surd$ \\
\cline{1-4} \cline{6-10}
\multicolumn{10}{c}{ } \\
\multicolumn{4}{c}{(a)} & \multicolumn{1}{c}{ } & \multicolumn{5}{c}{(b)} \\
\end{tabular}
\end{center}
\caption{Parameterized complexity analysis and intractability maps (Adapted from 
          Figure 1 in \cite{WarSub}). (a) A set of parameterized intractability 
          (R1, R2) and tractability (R3) results for a problem $\Pi$ relative to 
          parameter-set (A, B, C, D). (b) An intractability map
          derived from the results in (a). See main text for explanation. 
        }
\label{FigIMap} 
\end{figure}

An example derivation of an intractability map for a hypothetical problem $\Pi$ with parameter-set
$\{A, B, C, D\}$ is given in Figure \ref{FigIMap}. Part (a) of this figure describes a
set of parameterized intractability (R1, R2) and tractability (R3) results for $\Pi$;
note that tractability results are highlighted by boldfacing.
Each column in this table describes a result which holds relative to the parameter-set 
consisting of all parameters with a $@$-symbol in that column. If in addition a result 
holds when a particular parameter has a constant value $c$, that is indicated by $c$ 
replacing $@$ for that parameter in that result's column. Part (b) gives the 
intractability map associated with the results in part (a). Each cell in this map 
denotes the parameterized status of $\Pi$ (X for fp-intractability, $\surd$ for 
fp-tractability) relative to by the union of the sets of parameters labelling that 
cell's column and row. The ``raw'' results from the table in part (a) are denoted by 
superscripted entries ($X^{R1}$, $X^{R2}$, $\surd^{R3}$) and all other $X$ and $\surd$ 
results in the map follow from these observations:

\begin{lemma}
\cite[Lemma 2.1.30]{War99}
If problem $\Pi$ is fp-tractable relative to \newline parameter-set $K$ then $\Pi$ is
       fp-tractable for any parameter-set $K'$ such that $K \subset K'$.
\label{LemPrmProp1}
\end{lemma}

\begin{lemma}
\cite[Lemma 2.1.31]{War99}
If problem $\Pi$ is fp-intractable relative to \newline parameter-set $K$ then $\Pi$ is
       fp-intractable for any parameter-set $K'$ such that $K' \subset K$.
\label{LemPrmProp2}
\end{lemma}

\noindent
The remaining ???-entries correspond to parameter-sets 
whose parameterized complexity is not specified or implied in the given results. 
As there are ???-entries in the map in part (b), this map is a partial
intractability map.

The effort involved in constructing an intractability map can be 
reduced (in some cases, dramatically) by applying the following two rules of thumb.
First, prove the initial polynomial-time intractability of the problem of interest
using reductions from problems like {\sc Dominating set} whose parameterized 
versions are known to be fp-intractable; this often allows such reductions to be 
re-used in the subsequent derivation of parameterized results. Second, in order to
exploit Lemmas \ref{LemPrmProp1} and \ref{LemPrmProp2} to maximum effect when 
filling in the intractability map, fp-intractability results should be proved 
relative to the largest possible sets of parameters and fp-tractability results 
should be proved relative to the smallest possible sets of parameters. Both of 
these rules are used to good effect in the parameterized complexity analysis of 
learning from demonstration given in the next section.

\section{Case Study: Learning from demonstration (LfD)}

\label{SectCS}

Learning from demonstration (LfD) \cite{AC+09,BC+08} is a popular approach for deriving
policies that specify what action should be performed next given the current state of 
the environment. In LfD, a policy is derived from a set of one or more demonstrations, 
each of which is
a sequence of one or more of environment-state / action pairs. LfD can be
used by itself or as a generator of initial policies that are optimized by
techniques like reinforcement learning \cite[Sections 5.1 and 5.2]{HG+17}.

A number of
algorithms and systems implementing LfD have been proposed over the last 35 years 
(see Section 59.2 of \cite{BC+08} and Section 4 of \cite{HG+17}). Some of 
these systems operate in ``batch mode'' \cite[Page 471]{AC+09}, i.e., a policy 
is derived from a given set of demonstrations, while others are incremental
\cite[Section 59.3.2]{BC+08}, i.e., a policy derived relative to a set of 
previously-encountered demonstrations (which may or may not still be available) is
modified to take into account a new demonstration. Fast learning relative to few
demonstrations is often desirable \cite[Page 475]{AC+09} and in some situations 
necessary \cite[Page 5]{HG+17}. However, it is not known if existing (or indeed any) 
LfD algorithms can perform fast and reliable learning or, if not, under 
which restrictions such learning is possible.

In this section, we shall illustrate how the classical and parameterized complexity 
analysis techniques described in Section \ref{SectPCA} can be applied to answer these
questions relative to basic formalizations of batch and incremental LfD relative
to memoryless reactive policies given in Section \ref{SectCSForm}. We 
prove that all of these problems are polynomial-time intractable (and inapproximable)
in general (Section \ref{SectCSCIRes}) and remain fixed-parameter intractable under
a number of (often simultaneous) restrictions (Section \ref{SectCSPrmRes}). 
The implications of all of our results for both for the basic conception of
LfD examined here and LfD in practice are then discussed in Section \ref{SectDisc}.

\subsection{Formalizing LfD}

\label{SectCSForm}

In order to perform our computational complexity analyses, we must first formalize the 
following entities and properties associated with learning from demonstration:

\begin{itemize}
\item Sensed environmental features and environmental states;
\item Demonstrations of activities to be learned;
\item Policies describing actions taken by robots in particular situations;
\item What it means for a policy to correctly describe and hence be consistent
       with a given demonstration; and
\item What it means for a policy $p$ be be behaviorally equivalent to and hence
       consistent with another policy $p'$ that has been derived from $p$.
\end{itemize}

\noindent
We will then formalize problems corresponding to batch and incremental versions of LfD.
As part of our formalization process, we shall
also discuss how our formalizations compare with those given to date in the 
literature.

We first formalize the basic entities associated with LfD:

\begin{itemize}
\item{
{\em Sensed environmental features and environmental states}:
Let $F = \{f_1. f_2, \ldots,$ \linebreak $f_{|F|}\}$ be a set of features that a robot 
can sense
in its environment. A state $s \in 2^{F}$ of the environment is represented by the
subset of sensed features that characterize that state. Our features can be viewed as
special cases of both Boolean-predicate and other multi-valued features 
\cite[Section 3]{HG+17} in which the presence of a feature in a state corresponds to that
feature being true or having a particular feature-value, respectively. 
An example feature-set $F = \{f_1, f_2, f_3, f_4\}$ is given in part (a) of
Figure \ref{FigLfDEx}.
}
\begin{figure}[p]
\begin{center}
\large
\begin{tabular}{l p{0.6in} l}
$f_1 =$ weather is raining & & $a_1 =$ wear a raincoat \\
$f_2 =$ weather is cold    & & $a_2 =$ wear a sweater \\
$f_3 =$ weather is sunny   & & $a_3 =$ wear sunglasses \\
$f_4 =$ weather is windy   & & $a_4 =$ wear a wingsuit \\
 & & \\
\multicolumn{1}{c}{(a)} & & \multicolumn{1}{c}{(b)} \\
 & & \\ & & \\
\end{tabular}
\begin{tabular}{l}
$d_1 = ( pos, \la (\{f_1\}, a_1), (\{f_2\}, a_2), (\{f_3\}, a_3), (\{f_4\}, a_4)\ra)$  \\
$d_2 = (neg, \la (\{f_1\}, a_2), (\{f_3\}, a_1)\ra)$ \\
$d_3 = (pos, \la (\{f_1, f_4\}, a_1), (\{f_2, f_4\}, a_4), (\{f_3\}, a_2)\ra)$ \\
$d_4 = (pos, \la (\{f_2\}, a_2), (\{f_3\}, a_3)\ra)$ \\
\\
\multicolumn{1}{c}{(c)} \\
\\
\\
\end{tabular}
\begin{tabular}{l}
$p_1: T_1 = \{(\{f_1\}, a_1), (\{f_2\}, a_2), (\{f_3\}, a_3), (\{f_4\}, a_4)\}$ \\
$p_2: T_2 = \{(\{f_1\}, a_1), (\{f_2\}, a_4), (\{f_3\}, a_2)\}$ \\
\\
\multicolumn{1}{c}{(d)} \\
\end{tabular}
\end{center}
\caption{Examples of basic LfD entities in our model. a) A set $F = \{f_1, f_2, f_3, 
          f_4\}$ of sensed environmental features. b) a set $A = \{a_1, a_2, a_3, a_4\}$
          of actions. c) A set $D = \{d_1, d_2, d_3, d_4\}$ of demonstrations.
          d) Two policies $p_1$ and $p_2$ based on transition-sets $t_1$ and $t_2$,
          respectively.}
\label{FigLfDEx}
\end{figure}

\item{
{\em Demonstrations}:
A demonstration $d = (type, \la (s_1, a_1), (s_2, a_2), \ldots, (s_{|d|}, a_{|d|})\ra)$ 
consists of a demonstration-type $type \in \{pos, neg\}$ and a sequence of one or more
environment-state / action pairs whose actions are drawn from an action-set $A$. If $d$
is of type $pos$, $d$ is a positive demonstration; otherwise, $d$ is a negative 
demonstration. As such, our demonstrations are based on state-spaces and actions that 
are discrete \cite[Section 4.1]{AC+09}. Our positive demonstrations are as standardly 
defined for discrete actions \cite{AC+09}; however, our negative demonstrations are 
special cases of the negative demonstrations in \cite{ABV07,NM03} as our 
negative demonstrations forbid all state / action pairs in those demonstrations rather 
than specific state / action pairs in a demonstration sequence. A set of demonstrations
$D = \{d_1, d_2, d_3, d_4\}$ based on feature-set $F$ and action-set $A$ given in
parts (a) and (b) of Figure \ref{FigLfDEx}, respectively, is given in part (c) of
Figure \ref{FigLfDEx}.
}
\item{
{\em Policies}:
We will consider here the simplest possible type of reactive mapping-function policy 
\cite[Section 4,1]{AC+09} stated in terms of
a single-state transducer consisting of a state $q$ and a transition-set 
$T = ((F^t_1,a_1), (F^t_2,a_2), \ldots, (F^t_{|T|},a_{|T|}))$ where the action $a_i$
in each transition is drawn from an action-set $A$. Given an environment-state $s \in 
2^F$, we say that a transition $(F^t,a)$ triggers on $s$ if $F^t \subseteq s$. Our 
transition-triggering feature-sets are special cases of transition-triggering patterns 
encoded as Boolean formulas over environmental features \cite[e.g.,]{DL84} in which our
feature-sets correspond to patterns composed of AND-ed sets of features. As our policies
are not dependent on time, they are stationary (autonomous) \cite[Section 2]{HG+17}. Two
example policies $p_1$ and $p_2$ based on the feature-set $F$ and action-set $A$ given
in parts (a) and (b) of Figure \ref{FigLfDEx}, respectively, are given in part (d) of
Figure \ref{FigLfDEx}.

The behaviour of our policies should also be simple. To this end,
we adopt a conservative notion of generalization to previously unobserved 
environmental states, in that a policy $p$ is not defined and hence does not
produce an action for any state $s$ for which there is no transition $(F^t, a)$ in $p$
such that $F^t \subseteq s$. We also adopt a conservative notion of policy
determinism, in that a policy is only guaranteed to be deterministic and produce a
single action relative to states encoded in a demonstration-set $D$. More formally, the
transition-set $T$ in any $p$ relative to a demonstration-set $D$ is such that for each
state $s$ in a demonstration in $D$, all transitions in $T$ that trigger on $s$ 
produce the same action $a$. Such a policy $p$ is said to be 
{\em valid} relative to $D$. 

In some situations it will be useful to derive one policy from another. Given two 
policies $p$ and $p'$, we say that $p'$ is derivable from $p$ by at most $c$ changes if
at most $c$ modifications drawn from the set $\{${\em substitute new transition 
feature-set, substitute new transition action, delete transition}$\}$ are required to 
transform $p$ into $p'$.
}
\end{itemize}

We now formalize the following properties associated with LfD:

\begin{itemize}
\item{
{\em Policy-demonstration consistency}:
Given a policy $p$ and a demonstration $d$, $p$ is consistent with $d$ if,
starting from state $q$ and environment-state $s_1$, either 

\begin{enumerate}
\item for each of the state-action-pairs $(s, a)$ in $d$, $p$ produces $a$ when
       run on $s$ (if $d$ is a positive demonstration) or
\item for each of the state-action-pairs $(s, a)$ in $d$, $p$ does not produce $a$ when
       run on $s$ (if $d$ is a negative demonstration).
\end{enumerate}

\noindent
A policy $p$ is in turn consistent with a demonstration-set $D$ if $p$ is consistent
with each demonstration $d \in D$.
Consistency of a policy with a positive demonstration is as standardly defined
\cite{AC+09,HG+17}. Consistency of a policy with a negative demonstration is a special
case of such consistency as defined in \cite{ABV07,NM03} and is in line 
with our special-case definition of negative demonstrations given above.
}
\item{
{\em Policy-policy consistency}: 
This has not to our knowledge been previously defined in the literature; however,
the following version will be of use in defining and analyzing incremental LfD
when previously-learned demonstrations are not available.
Given two policies $p$ and $p'$ and a demonstration $d$, $p$ is consistent 
with $p'$ modulo $d$ if

\begin{enumerate}
\item for each $s$ that is the triggering feature-set of some transition in $p'$ such 
       that $s \not\subseteq s'$ for some state / action pair $(s',a)$ in $d$, $p$ and 
       $p'$ produce the same action when run on $s$; and
\item{for each $s$ in some state / action pair $(s,a)$ in $d$, either

\begin{enumerate}
\item $p$ produces action $a$ when run on $s$ (if $d$ is a positive demonstration) or
\item $p$ either (1) produces action $a'$ when run on $s$ if $p'$ produces action $a' 
       \neq a$ when run on $s$ or (2) $p$ does not produce an action for $s$ (if $d$
       is a negative demonstration).
\end{enumerate}
}
\end{enumerate}

\noindent
The focus above exclusively on the feature-sets in the transitions of $p'$ may 
initially seem strange. However, as any triggering feature-set of a transition 
$t$ in $p$ that mimics the behaviour of a transition $t'$ in $p'$ must have a 
triggering feature-set that is equal to or a subset of the triggering 
feature-set in $t'$, any such consistent $p$ will (modulo
the  behaviors requested or forbidden by $d$) replicate the behaviour of $p'$.
}
\end{itemize}

\noindent
Relative to these property formalizations, the following statements are true for
the examples of LfD entities given in Figure \ref{FigLfDEx}:

\begin{itemize}
\item {\em $p_1$ is valid for $\{d_1, d_2\}$}: This is so because for each 
       environment-state $s$ in $d_1$ and $d_2$, $p_1$ produces at most (and in
       all cases exactly) one action.
\item {\em $p_1$ is not valid for $\{d_3\}$}: This is so because $p_1$ produces
       action-sets $\{a_1, a_4\}$ and $\{a_2, a_4\}$ for environment-states
       $\{f_1, f_4\}$ and $\{f_2, f_4\}$ in $d_3$, respectively.
\item {\em $p_2$ is valid for $\{d_1, d_2\}$}: This is so because for each 
       environment-state $s$ in $d_1$ and $d_2$, $p_2$ produces at most one action.
\item {\em $p_2$ is valid for $\{d_3\}$}: This is so because for each 
       environment-state $s$ in $d_3$, $p_2$ produces at most one action.
\item {\em $p_2$ is derivable from $p_1$ by at most 3 changes}: If the transitions
       of $p_1$ in $T_1$ are numbered 1--4 as they appear in $T_1$, this is so
       by substituting transition action $a_4$ for $a_2$ in $t_2$, substituting 
       transition-action $a_2$ for $a_3$ in $t_3$, and deleting transition $t_4$.
\item {\em $p_1$ is not derivable from $p_2$ by any number of changes}: This is so
       because $p_1$ has more transitions than $p_2$ and adding transitions is not an 
       allowed change.
\item {\em $p_1$ is consistent with $\{d_1, d_2\}$}: This is because $p_1$
       produces all requested actions when run on the environment-states in $d_1$
       and none of the specified actions when run on the environment-states in $d_2$.
\item {\em $p_1$ is not consistent with $\{d_3\}$}: This is because $p_1$ produces
       two-action instead of one-action sets for the environment-states
       $\{f_1, f_4\}$ and $\{f_2, f_4\}$ in $d_3$ and the wrong action for 
       environment-state $\{f_3\}$ in $d_3$.
\item {\em $p_2$ is not consistent with $\{d_1, d_2\}$}: This is so because 
       $p_2$ produces the wrong actions for environment-states $\{f_2\}$, $\{f_3\}$,
       and $\{f_4\}$ in $d_1$ (with the last of these actually producing no action
       at all). Note that $p_2$ is, however, consistent with $\{d_2\}$.
\item {\em $p_2$ is consistent with $\{d_3\}$}: This is because $p_2$
       produces all requested actions when run on the environment-states in $d_3$.
\item {\em $p_1$ is consistent with $p_2$ modulo $\{d_4\}$}: This is so because
       $p_1$ and $p_2$ produce the same actions for all transition trigger-sets in
       $p_2$ that are not subsets (proper or otherwise) of environment-states in
       $d_4$ (namely, $\{f_1\}$) and for all other environment-states in
       $d_4$ (namely $\{f_2\}$ and $\{f_3\}$), $p_1$ produces the action requested
       in $d_4$.
\item {\em $p_2$ is not consistent with $p_1$ modulo $\{d_4\}$}: This is so because
       $p_2$ does not produce the same action as $p_1$ when run on the transition
       trigger-set $\{f_4\}$ in $p_1$.
\end{itemize}

Given all of the above, we can formalize the LfD problems 
that we will analyze in the remainder of this paper:

\newpage

\noindent
{\sc Batch Learning from Demonstration} (LfDBat) \\
{\em Input}: A set $D$ of demonstrations based on a feature-set $F$ and an action-set
              $A$ and positive non-zero integers $t$ and $f_t$. \\
{\em Output}: A policy $p$ valid for and consistent with $D$ such that there are at most
               $t$ transitions in $p$ and each transition is triggered by a set of at 
	       most $f_t$ features, if such a $p$ exists, and special symbol $\bot$ 
	       otherwise.

\vspace*{0.1in}

\noindent
{\sc Incremental Learning from Demonstration with History} (LfDIncHist) \\
{\em Input}: A set $D$ of demonstrations based on a feature-set $F$ and an action-set 
              $A$, a policy $p$ that is valid for and consistent with $D$,
              a demonstration $d_{new}$ based on $F$ and $A$ such that 
              $d_{new} \not\in D$, and positive non-zero integers $t$, $f_t$, and $c$. \\
{\em Output}: A policy $p'$ derivable from $p$ by at most $c$ changes that 
               is valid for and consistent with $D \cup \{d_{new}\}$ such that there are
               at most $t$ transitions in $p'$ and each transition is triggered by a
	       set of at most $f_t$ features, if such a $p'$ exists, and special symbol 
	       $\bot$ otherwise.

\vspace*{0.1in}

\noindent
{\sc Incremental Learning from Demonstration without History} (LfDIncNoHist) \\
{\em Input}: A policy $p$ based on a feature-set $F$ and an action-set $A$,
              a demonstration $d_{new}$ based on $F$ and $A$, and
              positive non-zero integers $t$, $f_t$, and $c$. \\
{\em Output}: A policy $p'$ derivable from $p$ by at most $c$ changes that 
               is consistent with $p$ modulo $d_{new}$ such that there are at most $t$ 
               transitions in $p'$ and each transition in $p'$ is triggered by
	       a set of at most $f_t$ features, if such a $p'$ exists, and special 
	       symbol $\bot$ otherwise.

\vspace*{0.1in}

\noindent
Let LfDIncHist$^{pos}$ and LfDIncHist$^{neg}$ (LfDIncNoHist$^{pos}$ and 
LfDIncNoHist$^{neg}$) denote the versions of LfDIncHist (LfDIncNoHist)
in which $d_{new}$ is a positive and negative demonstration, respectively;
furthermore, let LfDBat$_D$, LfDIncHist$^{pos}_D$, LfDIncHist$^{neg}_D$,
\linebreak
LfDIncNoHist$^{pos}_D$, and LfDIncNoHist$^{neg}_D$) denote the decision versions
of the problem above which ask if the requested policy exists. Some readers may be
disconcerted that we have incorporated explicit limits on the size and structure of
the requested policies. This is useful in practice for applications in which LfD 
must be done with limited computer memory \cite[Page 5]{HG+17}. This is also 
useful in allowing us to investigate the effects of various aspects 
of policy size and structure on the computational difficulty of LfD.

\subsection{LfD is Polynomial-time Intractable}

\label{SectCSCIRes}

Let us now revisit our first question of interest---namely, are there efficient 
algorithms for any of the LfD problems defined in Section \ref{SectCSForm} that are
guaranteed to always produce their requested policies? We will answer this question
 using polynomial-time reductions from the problem {\sc Dominating set}$_D$ defined
in Section \ref{SectPCA}. The following definitions, assumptions, and known results
will be useful below. For each vertex $v \in V$ in an instance of {\sc Dominating 
set}$_D$, let the complete neighbourhood $N_C(v)$ of $v$ be the set composed of $v$
and the set of all vertices in $G$ that are adjacent to $v$ by a single edge, i.e.,
$v \cup \{ u ~ | ~ u ~ \in V ~ \rm{and} ~ (u,v) \in E\}$. We assume for each 
instance of {\sc Dominating set}$_D$ an arbitrary ordering on the vertices of $V$ 
such that $V = \{v_1, v_2, \ldots, v_{|V|}\}$. Let {\sc Dominating set$^{PD3}_D$} 
denote the version of {\sc Dominating set}$_D$ in which the given graph $G$ is 
planar and each vertex in $G$ has degree at most 3. Both {\sc Dominating set}$_D$ 
and {\sc Dominating set$^{PD3}_D$} are $NP$-hard \cite[Problem GT2]{GJ79}.

\begin{lemma}
{\sc Dominating set}$_D$ polynomial-time reduces to LfDBat$_D$ such that in the 
constructed instance LfDBat$_D$, $|A| = \#d = f_t = 1$ and
$t$ is a function of $k$ in the given instance of {\sc Dominating set}$_D$.
\label{LemRedDS_LfDBat1}
\end{lemma}

\begin{proof}
Given an instance $\la G = (V,E), k \ra$ of {\sc Dominating Set}$_D$, construct an instance
$\la D, t, f_t \ra$ of LfDBat$_D$ as follows: Let $F = \{f_1, f_2, \ldots, f_{|V|}\}$, 
$A = \{a\}$, $D = \{((pos, \la(s_1, a)\ra), (pos, \la(s_2, a)\ra), \ldots, 
(pos, \la (s_{|V|}, a)\ra)\})$ where $s_i$ is the set consisting of the features in $F$
corresponding to the complete neighbourhood of $v_i$ in $G$, $t = k$, and $f_t 
= 1$. Observe that this construction can be done in  time polynomial in the size of the
given instance of {\sc Dominating set}$_D$.

We shall prove the correctness of this reduction in two parts. First, suppose that there
is a subset $V' = \{v'_1, v'_2, \ldots v'_l\} \subseteq V$, $l \leq k$, that is a 
dominating set in $G$. Construct a policy $p$ with a transition $(\{f_i\}, a)$ for
each $v'_i \in V'$. Observe that the number of transitions in $p$ is at most $t$ and
that $p$ is valid for $D$ (as all transitions produce the same action).
Moreover, as $V'$ is a dominating set in $G$ and the state in each
demonstration in $D$ corresponds to the complete neighbourhood of one of the vertices in
$G$, $p$ will produce the correct action for every demonstration in $D$ and hence
is consistent with $D$.

Conversely, suppose there is a policy $p$ consistent with $D$ such that there are at 
most $t$ transitions in $p$ and each transition is triggered by a set of at most $f_t$ 
features. As $f_t = 1$, the states in the demonstrations in $D$ correspond to the 
complete neighborhoods of the vertices in $G$, and $p$ is consistent with $D$, the set
of features labeling the transitions in $p$ corresponds to a dominating set in $G$. 
Moreover, as $t = k$, this dominating set is of size at most $k$.

To complete the proof, observe that in the constructed instance of LfDBat$_D$,
$|A| = \#d = f_t = 1$ and $t = k$.
\end{proof}

\begin{lemma}
{\sc Dominating set}$_D$ polynomial-time reduces to LfDIncHist$^{pos}_D$ such that in the 
constructed instance LfDIncHist$^{pos}_D$, $|A| = \#d = 2$, $f_t = 1$, and
$t$ and $c$ are functions of $k$ in the given instance of {\sc Dominating set}$_D$.
\label{LemRedDS_LfDIncHistPos1}
\end{lemma}

\begin{proof}
Given an instance $\la G = (V,E), k \ra$ of {\sc Dominating Set}$_D$, construct an instance
$\la D, p, d_{new}, c, t, f_t \ra$ of LfDIncHist$^{pos}_D$ as follows: Let $F = \{f_1,
f_2, \ldots, f_{|V|},$ $f_x, f_y\}$, $A = \{a_1, a_2\}$, and $D = \{(pos, 
\la(s_1, a_1)\ra), (pos, \la(s_2, a_1)\ra), \ldots,$ $(pos, \la (s_{|V|}, a_1)\ra)\}$ 
where $s_i$ is the set consisting of $f_x$ and the features in $F$ corresponding to 
the complete neighbourhood of $v_i$ in $G$. Let $p$ have $t = k + 1$ transitions, 
where the first transition is $(\{f_x\}, a_1)$ and the remaining $k$ transitions have 
the form $(\{f_i\}, a_1)$ where $f_i$ is the feature corresponding to a randomly 
selected vertex in $G$. Finally, let $d_{new} = (pos, \la (\{f_x, f_y\}, a_2)\ra)$, 
$c = k + 2$, and $f_t = 1$. Note that $p$ is valid for $D$ (as all transitions produce the 
same action) and consistent with $D$ (as the first transition in $T$ will always 
generate the correct action $a_1$ for each demonstration in $D$). Observe that this 
construction can be done in time polynomial in the size of the given instance of 
{\sc Dominating set}$_D$.

We shall prove the correctness of this reduction in two parts. First, suppose that there
is a subset $V' = \{v'_1, v'_2, \ldots v'_l\} \subseteq V$, $l \leq k$, that is a 
dominating set in $G$. Construct a policy $p'$ with $l + 1$ transitions in which
the first transition is $(\{f_y\}, a_2)$ and the subsequent $l$ transitions have the
form $(\{f_i\}, a_1)$ for each $v'_i \in V'$. Observe that $p'$ can be derived from $p$
by at most $c = k + 2$ changes to $p$ (namely, change the feature-set and action of
the first transition and the feature-sets of the next $l$ transitions as necessary
and delete the final $k - l$ transitions) and that $p'$ is valid for $D \cup 
\{d_{new}\}$ (as each state in the demonstrations in $D \cup \{d_{new}\}$ cause $p$ to 
produce at most one action). As $V'$ is a dominating set in $G$ and the state in each
demonstration in $D$ corresponds to the complete neighbourhood of one of the vertices in
$G$, $p'$ will produce the correct action for every demonstration in $D$ and hence
is consistent with $D$. Moreover, the first transition in $p'$ produces the correct 
action for $d_{new}$, which means that $p'$ is consistent with $D \cup \{d_{new}\}$.

Conversely, suppose there is a policy $p'$ derivable from $p$ by at most $c$ changes
that is valid for and consistent with $D \cup \{d_{new}\}$ and has $l \leq t = k + 
1$ transitions, each of which is triggered by a set of at most $f_t$ features. One of
these transitions must produce action $a_2$ in order for $p'$ to be consistent with
$d_{new}$; moreover, this transition must also trigger on feature-set $\{f_y\}$
(as triggering on $\{f_x\}$, the only other option to accommodate $d_{new}$, would
cause $p'$ to produce the wrong action for all demonstrations in $D$). As $f_y$ does 
not occur in any state in $D$, the remaining $l - 1$ transitions in $p'$ must
produce action $a_1$ for all states in $D$ for $p'$
to be consistent with $D$. As $f_t = 1$ and the states in the demonstrations in $D$ 
correspond to the complete neighborhoods of the vertices in $G$, the set of features 
triggering the final $l - 1$ transitions in $p'$ must correspond to a dominating set 
of size at most $k$ in $G$. 

To complete the proof, observe that in the constructed instance of  
LfDIncHist$^{pos}_D$, $|A| = \#d = 2$, $f_t = 1$, $c = k + 2$, and $t = k + 1$.
\end{proof}

\newpage

\noindent
The remaining reductions from {\sc Dominating set}$_D$ to LfDIncHist$^{neg}_D$, 
LfDIncNoHist$^{pos}_D$, and LfDIncNoHist$^{neg}_D$ are given in Lemmas 
\ref{LemRedDS_LfDIncHistNeg1}, \ref{LemRedDS_LfDIncNoHistPos1} and
\ref{LemRedDS_LfDIncNoHistNeg1} in the appendix.

\vspace*{0.15in}

\noindent
{\bf Result A}: LfDBat, LfDIncHist$^{pos}$, LfDIncHist$^{neg}$, LfDIncNoHist$^{pos}$, 
                 linebreak and LfDIncNoHist$^{neg}$ are not polynomial-time tractable
                 unless $P = NP$.

\vspace*{0.1in}

\begin{proof}
The $NP$-hardness of LfDBat$_D$, LfDIncHist$^{pos}_D$, LfDIncHist$^{neg}_D$, 
LfDIncNoHist$^{pos}_D$, 
and LfDIncNoHist$^{neg}_D$ follows from the $NP$-hardness of {\sc Dominating Set}$_D$ and 
the reductions in Lemmas \ref{LemRedDS_LfDBat1}, \ref{LemRedDS_LfDIncHistPos1},
\ref{LemRedDS_LfDIncHistNeg1}, \ref{LemRedDS_LfDIncNoHistPos1}, and
\ref{LemRedDS_LfDIncNoHistNeg1}. The result then follows from Lemma \ref{LemProp1}.
\end{proof}

\vspace*{0.15in}

\noindent
Given that the $P \neq NP$ conjecture is widely believed to be true 
\cite{For09,GJ79}, this establishes that the most common types of LfD cannot be done 
both efficiently and correctly for all inputs. 

\subsubsection{LfD is Also Polynomial-time Inapproximable}

Though it is not commonly known outside computational complexity circles, 
$NP$-hardness results such as those underlying Result A also imply various types of \linebreak
polynomial-time inapproximability. A polynomial-time approximation algorithm is an 
algorithm that runs in polynomial time in an approximately correct (but acceptable) 
manner for all inputs.
There are a number of ways in which an algorithm can operate in an
approximately correct manner. Three of the most popular ways are as follows:

\begin{enumerate}
\item {\bf Frequently Correct (Deterministic)} \cite{HW12}: Such an algorithm runs in
       polynomial
       time and gives correct solutions for all but a very small number of inputs. In
       particular, if the number of inputs for each input-size $n$ on which the
       algorithm gives the wrong or no answer (denoted by the function $err(n)$) is
       sufficiently small (e.g., $err(n) = c$ for some constant $c$), such
       algorithms may be acceptable.
\item {\bf Frequently Correct (Probabilistic)} \cite{MR10}: Such an algorithm (which is
       typically
       probabilistic) runs in polynomial time and gives correct solutions with high
       probability.  In particular, if the probability of correctness is $\geq
       2/3$ (and hence can be boosted by additional computations running in
       polynomial time to be correct with probability arbitrarily close to 1
       \cite[Section 5.2]{Wig07}), such algorithms may be acceptable.
\item {\bf Approximately Optimal} \cite{AC+99}: Such an algorithm $A$ runs in
       polynomial time and gives a solution $A(x)$ for an input $x$ whose value
       $v(A(x))$ is guaranteed to be within a multiplicative factor $f(|x|)$ of the
       value $v_{OPT}(x)$ of an optimal solution for $x$, i.e., $|v_{OPT}(x) -
       v(A(x))| \leq f(|x|) \times
       v_{OPT}(x)$ for any input $x$ for some function $f()$. A problem with such an
       algorithm is said to be polynomial-time $f(|x|)$-approximable. In particular, if
       $f(|x|)$ is a constant very close to 0 (meaning that the algorithm is always
       guaranteed to give a solution that is either optimal or very close to optimal),
       such algorithms may be acceptable.
\end{enumerate}

\noindent
It turns out that none of our LfD problems have such algorithms.

\vspace*{0.15in}

\noindent
{\bf Result B}: 
If LfDBat, LfDIncHist$^{pos}$, LfDIncHist$^{neg}$, LfDIncNoHist$^{pos}$, or  \linebreak
LfDIncNoHist$^{neg}$ is solvable by a polynomial-time algorithm with a polynomial error 
frequency (i.e., $err(n)$ is upper bounded by a polynomial of $n$) then $P = NP$.

\vspace*{0.1in}

\begin{proof}
That the existence of such an algorithm for the decision versions of any of our LfD 
problems implies $P = NP$ follows from the $NP$-hardness of these problems (which
is established in the proof of Result A) 
and Corollary 2.2. in \cite{HW12}. The result then follows from the fact that any
such algorithm for any of our LfD problems can be used to solve the decision version
of that problem.
\end{proof}

\vspace*{0.15in}

\noindent
The following holds relative to both the $P \neq NP$ and $P = BPP$ conjectures, the 
latter of which is also widely believed to be true \cite{CRT98,Wig07}.

\vspace*{0.15in}

\noindent
{\bf Result C}: 
If $P = BPP$ and LfDBat, LfDIncHist$^{pos}$, LfDIncHist$^{neg}$, LfDIncNoHist$^{pos}$, 
or LfDIncNoHist$^{neg}$ is polynomial-time solvable by a probabilistic algorithm which
operates correctly with probability $\geq 2/3$ then $P = NP$.

\vspace*{0.1in}

\begin{proof}
It is widely believed that $P = BPP$ \cite[Section 5.2]{Wig07} where $BPP$ is 
considered the most inclusive class of decision problems that can be efficiently solved
using probabilistic methods (in particular, methods whose probability of correctness is
$\geq 2/3$ and can thus be efficiently boosted to be arbitrarily close to one). Hence, 
if any of LfDBat, LfDIncHist$^{pos}$, LfDIncHist$^{neg}$, LfDIncNoHist$^{pos}$, or 
LfDIncNoHist$^{neg}$ has a probabilistic polynomial-time algorithm which operates 
correctly with probability $\geq 2/3$ then by the observation on which Lemma 
\ref{LemProp1} is based, their corresponding decision versions also have such algorithms
and are by definition in $BPP$. However, if $BPP = P$ and we know that all these 
decision versions are $NP$-hard by the proof of Result A, this would then imply by the 
definition of $NP$-hardness that $P = NP$, completing the result.
\end{proof}

\vspace*{0.15in}

\noindent
Certain inapproximability results follow not so much from $NP$-hardness as
approximability characteristics of the particular problems used to establish 
$NP$-hardness. For any of our LfD problems with name {\bf X},
let {\bf X}${}_{OPT}$ be the version of {\bf X} that returns the policy $p$
with the smallest possible value of $t$ (with this value
being $\infty$ if there is no such $p$).

\vspace*{0.15in}

\noindent
{\bf Result D}: 
For any of our LfD problems with name {\bf X},
if {\bf X}${}_{OPT}$ is polynomial-time $c$-approximable for any constant $c > 0$
then $P = NP$.

\vspace*{0.1in}

\begin{proof}
Let {\sc Dominating set}${}_{OPT}$ be the version of {\sc Dominating set} which returns
the size of the smallest dominating set in $G$---that is, the search version of 
{\sc Dominating set} defined ins Section \ref{SectPCA}. Observe that in the reductions 
in the proof of Result A, the size $k$ of a dominating set in $G$ in the given instance
of {\sc Dominating set} is always a linear function of $t$ in the constructed
instance of {\bf X}${}_D$ (either $k = t$ (Lemmas \ref{LemRedDS_LfDBat1},
\ref{LemRedDS_LfDIncHistNeg1}, and \ref{LemRedDS_LfDIncNoHistNeg1}) or $k = t - 1 $ 
(Lemmas \ref{LemRedDS_LfDIncHistPos1} and \ref{LemRedDS_LfDIncNoHistPos1})).
In the first case, this means that a polynomial-time $c$-approximation algorithm for 
{\bf X}${}_{OPT}$ for any constant $c$ implies the existence of a polynomial-time
$c$-approximation algorithm for {\sc Dominating set}${}_{OPT}$ In the second case,
this means that the existence of
a polynomial-time $c$-approximation algorithm for {\bf X}${}_{OPT}$ 
for any constant $c$ implies the existence of a polynomial-time
$2c$-approximation algorithm for {\sc Dominating set}${}_{OPT}$
(as $c \times t = c \times (k + 1) \leq 2c \times k$ for $k \geq 1$).
However, if {\sc Dominating set}${}_{OPT}$ has a polynomial-time $c$-approximation
algorithm for any constant $c > 0$ then $P = NP$ \cite{LY94}, completing the proof.
\end{proof}

\vspace*{0.15in}

\noindent
Results B-D are not directly relevant to the goals of this paper, as approximation
algorithms by definition are not reliable in the sense defined in Section
\ref{SectIntro}. However, these results are still of interest
for other reasons. For example, Result D suggests that it is very difficult to
efficiently obtain even approximately minimum-size policies using LfD, which gives
additional motivation for the parameterized analyses in the next section.

\subsection{What Makes LfD Fixed-parameter Tractable?}

\label{SectCSPrmRes}

We now turn to
the question of what restrictions make the LfD problems defined in Section 
\ref{SectCSForm} tractable, which we rephrase as what combinations of parameters 
make our problems fixed-parameter tractable. The parameters examined
in this paper are shown in Table \ref{TabPrm} and can be broken into three groups:

\begin{enumerate}
\item Restrictions on environments ($|F|, |A|$);
\item Restrictions on demonstrations ($\#d, |d|, f_{eap}$); and
\item Restrictions on policies ($t, f_t, c$).
\end{enumerate}

\noindent
We consider first what parameters do not yield fp-tractability. We will do this
by exploiting the polynomial-time reductions from {\sc Dominating set}$_D$ in
Section \ref{SectCSCIRes}, which also turn out to be parameterized reductions
from $\la k \ra$-{\sc Dominating set}$_D$.
In addition to the definitions, assumptions, and known results given
at the beginning of Section \ref{SectCSCIRes}, the fact that $\la k \ra$-{\sc Dominating set}$_D$ is $W[2]$-hard 
\cite{DF99} will be useful, as will the following consequence 
of our definitions of policies and consistency in Section \ref{SectCSPrmRes}.

\begin{table}[t]
\centering
\begin{tabular}{| c ||  l | c | }
\hline
Parameter   & Description & Applicability \\
\hline\hline
$|F|$       & \# environmental description features & All \\
\hline
$|A|$       & \# actions that can be taken in & All \\
            & \hspace*{0.1in} an environment & \\
\hline\hline
$\#d$       & \# given demonstrations &  All \\
\hline
$|d|$       & max \# environment-state action pairs & All \\
            & \hspace*{0,1in} in a demonstration & \\
\hline
$f_{eap}$   & max \# features in a demonstration & All \\
            & \hspace*{0,1in} environment-state & \\
\hline\hline
$t$         & max \# transitions in a policy transducer & All \\
\hline
$f_t$       & max \# features in a transition triggering & All \\
            & \hspace*{0,1in} pattern & \\
\hline
$c$         & Max \# changes to transform one policy & LfDInc* \\
            & \hspace*{0,1in} into another & \\
\hline
\end{tabular} 
\caption{
Parameters for learning from demonstration problems.
}
\label{TabPrm}
\end{table}

\begin{lemma}
Any policy $p$ is consistent with a demonstration-set $D$ consisting of $m \geq 1$
positive single state / pair demonstrations if and only if $p$ is consistent with a
demonstration-set $D'$ consisting of a single positive demonstration in which all
$m$ single state / action pairs in the demonstrations in $D$ have been placed
in an arbitrary order.
\label{LemPropDemo}
\end{lemma}

\begin{lemma}
{\sc Dominating set$^{PD3}_D$} polynomial-time reduces to LfDBat$_D$ such that in the 
constructed instance LfDBat$_D$, $|A| = \#d = f_t = 1$, $f_{eap} = 4$, and
$t$ is a function of $k$ in the given instance of {\sc Dominating set$^{PD3}_D$}.
\label{LemRedDS_LfDBat2}
\end{lemma}

\begin{proof}
As {\sc Dominating set$^{PD3}_D$} is a special case of {\sc Dominating set}$_D$, the reduction
in Lemma \ref{LemRedDS_LfDBat1} from {\sc Dominating set}$_D$ to LfDBat$_D$ is also a 
reduction from {\sc Dominating set$^{PD3}_D$} to LfDBat$_D$ that constructs instances of 
LfDBat$_D$ such that $|A| = \#d = f_t = 1$ and
$t$ is a function of $k$ in the given instance of {\sc Dominating set$^{PD3}_D$}.
To complete the proof, note that as the degree of each vertex in graph $G$
in the given instance of {\sc Dominating set$^{PD3}_D$} is at most 3, the size
of each complete vertex neighbourhood is of size at most 4, which means that
$f_{eap} = 4$ in each constructed instance of LfDBat$_D$.
\end{proof}

\vspace*{0.1in}

\noindent
{\bf Result E}: LfDBat is not fp-tractable relative to the following 
                 parameter-sets:

                \begin{quote}
                \begin{description}
                \item[a)] $\{ |A|, \#d, t, f_t\}$ 
                            when $|A| = \#d = f_t = 1$ (unless $FPT = W[1]$)
                \item[b)] $\{ |A|, |d|, t, f_t\}$
                            when $|A| = |d| = f_t = 1$ (unless $FPT = W[1]$)
                \item[c)] $\{ |A|, \#d, f_{eap}, f_t\}$
                            when $|A| = \#d = f_t = 1$ and $f_{eap} = 4$ (unless $P = NP$)
                \item[d)] $\{ |A|, |d|, f_{eap}, f_t\}$
                            when $|A| = |d| = f_t = 1$ and $f_{eap} = 4$ (unless $P = NP$)
                \end{description}
                \end{quote}

\begin{proof} \hspace*{0.5in} \\
\noindent
{\em Proof of part (a)}: Follows from the $W[2]$-hardness of 
$\la k \ra$-{\sc Dominating set}$_D$, the reduction in Lemma \ref{LemRedDS_LfDBat1},
the inclusion of $W[1]$ in $W[2]$, and the conjecture $FPT \neq W[1]$.

\noindent
{\em Proof of part (b)}: Follows from part (a) and Lemma \ref{LemPropDemo}.

\noindent
{\em Proof of part (c)}: Follows from the $NP$-hardness of {\sc Dominating set$^{PD3}_D$},
the reduction in Lemma \ref{LemRedDS_LfDBat2}, and Lemma \ref{LemProp2}.

\noindent
{\em Proof of part (d)}: Follows from part (c) and Lemma \ref{LemPropDemo}.
\end{proof}

\vspace*{0.15in}

\noindent
The following analogous results hold for our remaining LfD problems, and their 
proofs are given in the appendix.

\vspace*{0.15in}

\noindent
{\bf Result F}: LfDIncHist$^{pos}$ is not fp-tractable relative to the following 
                 parameter-sets:

                \begin{quote}
                \begin{description}
                \item[a)] $\{ |A|, \#d, t, f_t, c\}$
                            when $|A| = \#d = 2$ and $f_t = 1$ (unless $FPT = W[1]$)
                \item[b)] $\{ |A|, |d|, t, f_t, c\}$
                            when $|A| = |d| = 2$ and $f_t = 1$ (unless $FPT = W[1]$)
                \item[c)] $\{ |A|, \#d, f_{eap}, f_t\}$
                            when $|A| = \#d = 2$, $f_{eap} = 4$, and $f_t = 1$ 
                            (unless $P = NP$)
                \item[d)] $\{ |A|, |d|, f_{eap}, f_t\}$
                            when $|A| = |d| = 2$, $f_{eap} = 4$, and $f_t = 1$ 
                            (unless $P = NP$)
                \end{description}
                \end{quote}

\vspace*{0.15in}

\noindent
{\bf Result G}: LfDIncHist$^{neg}$ is not fp-tractable relative to the following 
                 parameter-sets:

                \begin{quote}
                \begin{description}
                \item[a)] $\{ |A|, \#d, t, f_t, c\}$
                            when $|A| = \#d = 2$ and $f_t = 1$ (unless $FPT = W[1]$)
                \item[b)] $\{ |A|, |d|, t, f_t, c\}$
                            when $|A| = |d| = 2$ and $f_t = 1$ (unless $FPT = W[1]$)
                \item[c)] $\{ |A|, \#d, f_{eap}, f_t\}$
                            when $|A| = \#d = 2$, $f_{eap} = 4$, and $f_t = 1$ 
                            (unless $P = NP$)
                \item[d)] $\{ |A|, |d|, f_{eap}, f_t\}$
                            when $|A| = |d| = 2$, $f_{eap} = 4$, and $f_t = 1$ 
                            (unless $P = NP$)
                \end{description}
                \end{quote}

\vspace*{0.15in}

\noindent
{\bf Result H}: LfDIncNoHist$^{pos}$ is not fp-tractable relative to the following 
                 parameter-sets:

                \begin{quote}
                \begin{description}
                \item[a)] $\{ |A|, \#d, t, f_t\}$
                            when $|A| = \#d = f_t = 1$ (unless $FPT = W[1]$)
                \item[b)] $\{ |A|, |d|, t, f_t\}$
                            when $|A| = |d| = f_t = 1$ (unless $FPT = W[1]$)
                \item[c)] $\{ |A|, \#d, f_{eap}, f_t\}$
                            when $|A| = \#d = f_t = 1$ and $f_{eap} = 4$ 
                             (unless $P = NP$)
                \item[d)] $\{ |A|, |d|, f_{eap}, f_t\}$
                            when $|A| = |d| = f_t = 1$ and $f_{eap} = 4$ 
                             (unless $P = NP$)
                \end{description}
                \end{quote}

\vspace*{0.15in}

\noindent
{\bf Result I}: LfDIncNoHist$^{neg}$ is not fp-tractable relative to the following 
                 parameter-sets:

                \begin{quote}
                \begin{description}
                \item[a)] $\{ |A|, \#d, t, f_t\}$
                            when $|A| = \#d = f_t = 1$ (unless $FPT = W[1]$)
                \item[b)] $\{ |A|, |d|, t, f_t\}$
                            when $|A| = |d| = f_t = 1$ (unless $FPT = W[1]$)
                \item[c)] $\{ |A|, \#d, f_{eap}, f_t\}$
                            when $|A| = \#d = f_t = 1$ and $f_{eap} = 4$ 
                             (unless $P = NP$)
                \item[d)] $\{ |A|, |d|, f_{eap}, f_t\}$
                            when $|A| = |d| = f_t = 1$ and $f_{eap} = 4$ 
                             (unless $P = NP$)
                \end{description}
                \end{quote}

\vspace*{0.15in}

\noindent
Given that the $P \neq NP$ and $FPT \neq W[1]$ conjectures are widely believed
to be true \cite{DF99, DF13, For09,GJ79}, these results show that LfD
cannot be done efficiently under a number of restrictions. These results 
are more powerful than they first appear courtesy of Lemma \ref{LemPrmProp1},
which establishes in conjunction with these results that none of the parameters 
considered here except $|F|$
can be either individually or in many combinations be restricted 
to yield tractability for any of our LfD problems. Moreover, this intractability
frequently holds when these parameters are restricted to very small constant
values (see Tables \ref{TabPrmRes1} and \ref{TabPrmRes2} for 
details).

\begin{table}[p]
\centering
\small
\begin{tabular}{ | p{0.7cm} || p{0.65cm}  | p{0.65cm} | p{0.65cm} | p{0.65cm}  || p{0.65cm}  | p{0.65cm} |}
\multicolumn{7}{c}{   } \\
\multicolumn{7}{c}{(a)} \\
\multicolumn{7}{c}{   } \\
\hline
& \multicolumn{6}{c|}{LfDBat} \\
\cline{2-7}
& E(a) & E(b) & C(c) & E(d) & {\bf J} & {\bf K} \\
\hline\hline
$|F|$       & -- & -- & -- & -- & {\bf @} & {\bf --} \\ 
\hline
$|A|$       & 1  & 1  & 1  & 1  & {\bf @} & {\bf @} \\
\hline\hline
$\#d$       & 1  & -- & 1  & -- & {\bf --} & {\bf @} \\
\hline
$|d|$       & -- & 1  & -- & 1  & {\bf --} & {\bf @} \\
\hline
$f_{eap}$   & -- & -- & 4  & 4  & {\bf --} & {\bf @} \\
\hline\hline
$t$         & @  & @  & -- & -- & {\bf --} & {\bf --} \\
\hline
$f_t$       & 1  & 1  & 1  & 1  & {\bf --} & {\bf --} \\
\hline
$c$         & N/A & N/A & N/A & N/A & {\bf N/A} & {\bf N/A} \\
\hline
\multicolumn{7}{c}{   } \\
\multicolumn{7}{c}{(b)} \\
\multicolumn{7}{c}{   } \\
\hline
& \multicolumn{6}{c|}{LfDIncHist$^{pos}$} \\
\cline{2-7}
& F(a) & F(b) & F(c) & F(d) & {\bf J} & {\bf K} \\
\hline\hline
$|F|$       & -- & -- & -- & -- & {\bf @} & {\bf --} \\ 
\hline
$|A|$       & 2  & 2  & 2  & 2  & {\bf @} & {\bf @} \\
\hline\hline
$\#d$       & 2  & -- & 2  & -- & {\bf --} & {\bf @} \\
\hline
$|d|$       & -- & 1  & -- & 1  & {\bf --} & {\bf @} \\
\hline
$f_{eap}$   & -- & -- & 4  & 4  & {\bf --} & {\bf @} \\
\hline\hline
$t$         & @  & @  & -- & -- & {\bf --} & {\bf --} \\
\hline
$f_t$       & 1  & 1  & 1  & 1  & {\bf --} & {\bf --} \\
\hline
$c$         & @  & @  & -- & -- & {\bf --} & {\bf --} \\
\hline
\multicolumn{7}{c}{   } \\
\multicolumn{7}{c}{(c)} \\
\multicolumn{7}{c}{   } \\
\hline
& \multicolumn{6}{c|}{LfDIncHist$^{neg}$} \\
\cline{2-7}
& G(a) & G(b) & G(c) & G(d) & {\bf J} & {\bf K} \\
\hline\hline
$|F|$       & -- & -- & -- & -- & {\bf @} & {\bf --} \\ 
\hline
$|A|$       & 1  & 1  & 1  & 1  & {\bf @} & {\bf @} \\
\hline\hline
$\#d$       & 2  & -- & 2  & -- & {\bf --} & {\bf @} \\
\hline
$|d|$       & -- & 1  & -- & 1  & {\bf --} & {\bf @} \\
\hline
$f_{eap}$   & -- & -- & 4  & 4  & {\bf --} & {\bf @} \\
\hline\hline
$t$         & @  & @  & -- & -- & {\bf --} & {\bf --} \\
\hline
$f_t$       & 1  & 1  & 1  & 1  & {\bf --} & {\bf --} \\
\hline
$c$         & @  & @  & -- & -- & {\bf --} & {\bf --} \\
\hline
\end{tabular}
\caption{
A detailed summary of our parameterized complexity results. a) Results
          for LfDBat.  b) Results for LfDIncHist$^{pos}$. c) Results for 
	  LfDIncHist$^{neg}$.  
\label{TabPrmRes1}
}
\end{table}

Despite this, there are combinations of parameters relative to which our problems are
fp-tractable. 

\vspace*{0.15in}

\noindent
{\bf Result J}: LfDBat, LfDIncHist$^{pos}$, LfDIncHist$^{neg}$, LfDIncNoHist$^{pos}$, 
\linebreak and LfDIncNoHist$^{neg}$ are all fp-tractable relative 
to parameter-set $\{|F|, |A|\}$.

\vspace*{0.1in}

\begin{proof}
Consider the following algorithm for LfDBat: Generate all possible policies with
$t$ transitions triggered by feature-sets with at most $f_t$ features and for
each such policy $p$, determine if $p$ is valid for and consistent with $D$. 
There are $ t \leq (2^{|F|} - 1)|A|$ possible transitions and at most 
$\sum^t_{i=1} \left( \begin{array}{c} (2^{|F|} - 1)|A| \\ i \end{array} \right) \leq
 \sum^t_{i=1} (2^{|F|} - 1)|A|^i \leq
 \sum^t_{i=1} (2^{|F|} - 1)|A|^t =
t((2^{|F|} - 1)|A|)^t$
ways of choosing at most $t$ transitions to form a policy. Each such policy can be
checked in low-order polynomial time to see if each transition is triggered by a set of
at most $f_t$ features; moreover, it can also be verified in low-order polynomial time 
if each such $f_t$-limited policy is valid for and consistent with $D$ by running each 
environment-state in $D$ against all transitions in $p$ to see if the action associated
with that state is generated by all transitions in $p$ that are triggered by that state.
This algorithm thus runs in time upper-bounded by some function of $|F|$ and $|A|$ times
some polynomial of the instance size, which completes the proof for LfDBat. 

The algorithm above for $LfDBat$ can be used to find possible policies that are valid 
for and consistent with $D \cup \{d_{new}\}$ in LfDIncHist$^{pos}$ and 
LfDIncHist$^{neg}$, with the addition of a step that checks each such policy $p$ to see
if it is also derivable from $p'$ by at most $c$ changes. There are $t' \leq 
(2^{|F|} - 1$ transitions in $p'$ and hence at most $t'$ transitions can be deleted from
$p'$ to leave $t$ transitions in $p$. As there are at most $t \leq 2^{|F|} - 1$
transitions in $p$, there are at most $t$ transitions whose feature-sets
or actions can be substituted; moreover, there are $(2^{|F|} - 1) + |A|$ choices of
such substitutions. The number of choices of at most $c$ transition deletions and 
substitutions that can be made to $p'$ to create $p$ is thus at most


\begin{eqnarray*}
\sum^c_{i = 1} \sum^i_{j = 1} 
    \left( \begin{array}{c} t' \\ j \end{array} \right) 
    \left( \begin{array}{c} t \\ i - j \end{array} \right) 
    \left( \begin{array}{c} (2^{|F|} - 1) + |A| \\ i - j \end{array} \right) & \leq &
\sum^c_{i = 1} \sum^i_{j = 1} t'^j t^{i - j} ((2^{|F|} - 1) + |A|)^{i - j} \\
 & \leq & \sum^c_{i = 1} \sum^i_{j = 1} t'^i t^i ((2^{|F|} - 1) + |A|)^i \\
 & \leq & \sum^c_{i = 1} \sum^c_{j = 1} t'^c t^c ((2^{|F|} - 1) + |A|)^c \\
 & =    & \sum^c_{i = 1} c(t'^c t^c ((2^{|F|} - 1) + |A|)^c) \\
 & =    & c^2(t'^c t^c ((2^{|F|} - 1) + |A|)^c) \\
\end{eqnarray*}

\noindent
Each of these choices of deletions and substitutions can be applied in time 
polynomial in $c$ to see if they do in fact transform $p'$ into $p$. As $c \leq 
(2^{|F|} - 1) + (2^{|F|} - 1)$, this algorithm thus runs in time upper-bounded by some 
function of $|F|$ and $|A|$ times some polynomial of the instance size, which completes
the proof for LfDIncHist$^{pos}$ and LfDIncHist$^{neg}$.

Finally, the
algorithms above for LfDIncHist$^{pos}$ and LfDIncHist$^{neg}$  can be used to find 
possible policies $p'$ that are consistent with the given $p$ modulo $d_{new}$ if
one eliminates the step checking if $p'$ is consistent with $D \cup \{d_{new}\}$
and adds a step that checks each such policy $p'$ is consistent with $p$ modulo
$d_{new}$. As this new step only involves straightforward low-order polynomial
operations on the transitions in $p$ and $p'$ and the environment-state / action pairs
in $d_{new}$, this step can be done in low-order polynomial time. Hence,
the revised algorithm runs in time upper-bounded by some function 
of $|F|$ and $|A|$ times some polynomial of the instance size, which completes the 
proof for LfDIncNoHist$^{pos}$ and LfDIncNoHist$^{neg}$ as well as the proof as a whole.
\end{proof}

\begin{table}[p]
\centering
\begin{tabular}{ | p{0.7cm} || p{0.65cm}  | p{0.65cm} | p{0.65cm} | p{0.65cm}  || p{0.65cm} |}
\multicolumn{6}{c}{   } \\
\multicolumn{6}{c}{(d)} \\
\multicolumn{6}{c}{   } \\
\hline
& \multicolumn{5}{c|}{LfDIncNoHist$^{pos}$} \\
\cline{2-6}
& H(a) & H(b) & H(c) & H(d) & {\bf J} \\
\hline\hline
$|F|$       & -- & -- & -- & -- & {\bf @} \\ 
\hline
$|A|$       & 1  & 1  & 1  & 1  & {\bf @} \\
\hline\hline
$\#d$       & 1  & -- & 1  & -- & {\bf --} \\
\hline
$|d|$       & -- & 1  & -- & 1  & {\bf --} \\
\hline
$f_{eap}$   & -- & -- & 4  & 4  & {\bf --} \\
\hline\hline
$t$         & @  & @  & -- & -- & {\bf --} \\
\hline
$f_t$       & 1  & 1  & 1  & 1  & {\bf --} \\
\hline
$c$         & -- & -- & -- & -- & {\bf --} \\
\hline
\multicolumn{6}{c}{   } \\
\multicolumn{6}{c}{(e)} \\
\multicolumn{6}{c}{   } \\
\hline
& \multicolumn{5}{c|}{LfDIncNoHist$^{pos}$} \\
\cline{2-6}
& I(a) & I(b) & I(c) & I(d) & {\bf J} \\
\hline\hline
$|F|$       & -- & -- & -- & -- & {\bf @} \\ 
\hline
$|A|$       & 1  & 1  & 1  & 1  & {\bf @} \\
\hline\hline
$\#d$       & 1  & -- & 1  & -- & {\bf --} \\
\hline
$|d|$       & -- & 1  & -- & 1  & {\bf --} \\
\hline
$f_{eap}$   & -- & -- & 4  & 4  & {\bf --} \\
\hline\hline
$t$         & @  & @  & -- & -- & {\bf --} \\
\hline
$f_t$       & 1  & 1  & 1  & 1  & {\bf --} \\
\hline
$c$         & -- & -- & -- & -- & {\bf --} \\
\hline
\end{tabular}
\caption{
A detailed summary of our parameterized complexity results (Cont'd). d) Results
  for LfDIncNoHist$^{pos}$. e) Results for LfDIncNoHist$^{neg}$.
\label{TabPrmRes2}
}
\end{table}

\vspace*{0.15in}

\noindent
{\bf Result K}: LfDBat, LfDIncHist$^{pos}$, and LfDIncHist$^{neg}$
are all fp-tractable relative to parameter-set $\{|A|, \#d, |d|, f_{eap}\}$.

\vspace*{0.1in}

\begin{proof}
Follows from 
the algorithms given in the proof of Result H and the observation that, as both
$p$ and $p'$ can only use environment features in $D$ and $d_{new}$, $|F| \leq \#d 
\times |d| \times f_{eap}$.
\end{proof}

\vspace*{0.15in}

\noindent
Again, these results are more powerful than they first appear courtesy of Lemma
\ref{LemPrmProp2}, which establishes in conjunction with these results that
any set of parameters including both
$|F|$ and $|A|$ or all of $|A|$, $\#d$, $|d|$ and $f_{eap}$ can be restricted to yield 
tractability for all of our LfD problems and all of LfDBat, LfDIncHist$^{pos}$, and
LfDIncHist$^{neg}$, respectively.

\begin{table}[t]
\centering
\begin{tabular}{| c || c | c | c | c | c | c | c | c |}
\hline
             & --- & $\#d$   & $|d|$  & $f_{eap}$
             & $\#d, |d|$ & $\#d, f_{eap}$ & $|d|, f_{eap}$ 
             & $\#d, |d|, f_{eap}$ \\
\hline \hline
---                & NPh     & X       & X       & X           
                   & ???     & X       & X       & ???           \\
\hline
$t$                & X       & X       & X       & ???         
                   & ???     & ???     & ???     & ???           \\
\hline
$f_t$              & X       & X       & X       & X           
                   & ???     & X       & X       & ???           \\
\hline
$|F|$              & ???     & ???     & ???     & ???         
                   & ???     & ???     & ???     & ???           \\
\hline
$|A|$              & X       & X       & X       & X           
                   & ???     & X       & X       & $\surd$       \\
\hline
$t, f_t$           & X       & X       & X       & ???         
                   & ???     & ???     & ???     & ???           \\
\hline
$t, |F|$           & ???     & ???     & ???     & ???         
                   & ???     & ???     & ???     & ???           \\
\hline
$t, |A|$           & X       & X       & X       & ???         
                   & ???     & ???     & ???     & $\surd$       \\
\hline
$f_t, |F|$         & ???     & ???     & ???     & ???         
                   & ???     & ???     & ???     & ???           \\
\hline
$f_t, |A|$         & X       & X       & X       & ???         
                   & ???     & X       & X       & $\surd$       \\
\hline
$|F|, |A|$         & $\surd$ & $\surd$ & $\surd$ & $\surd$      
                   & $\surd$ & $\surd$ & $\surd$ & $\surd$       \\
\hline
$t, f_t, |F|$      & ???     & ???     & ???     & ???         
                   & ???     & ???     & ???     & ???           \\
\hline
$t, f_t, |A|$      & X       & X       & X       & ???         
                   & ???     & ???     & ???     & $\surd$       \\
\hline
$t, |F|, |A|$      & $\surd$ & $\surd$ & $\surd$ & $\surd$      
                   & $\surd$ & $\surd$ & $\surd$ & $\surd$       \\
\hline
$f_t, |F|, |A|$    & $\surd$ & $\surd$ & $\surd$ & $\surd$      
                   & $\surd$ & $\surd$ & $\surd$ & $\surd$       \\
\hline
$t, f_t, |F|, |A|$ & $\surd$ & $\surd$ & $\surd$ & $\surd$      
                   & $\surd$ & $\surd$ & $\surd$ & $\surd$       \\
\hline
\end{tabular}
\caption{Current Intractability Map for LfDBat.}
\label{TabIMLfDBat}
\end{table}

The intractability maps for our LfD problems relative to the results given above
are large, and hence we here only give the intractability map for LfDBat 
(Table \ref{TabIMLfDBat}). Though this map is partial, it has 70 out of 127
parameterized result cells filled, and this was accomplished courtesy of
Lemmas \ref{LemPrmProp1} and \ref{LemPrmProp2} using only 4 fp-intractability and
2 fp-tractability results. This very nicely demonstrates how the
rules of thumb for performing parameterized complexity analyses at the
end of Section \ref{SectPCA} can help minimize the effort of performing
these analyses.

\section{Discussion}

\label{SectDisc}

What do our results mean for LfD relative to the basic model investigated here? In 
addition to proving that batch LfD is not polynomial-time solvable in general, we have 
also shown that incremental LfD is not polynomial-time solvable in general when 
previously-encountered demonstrations either are or are not available. This 
intractability holds whether we require that requested policies are always output 
correctly (Result A) or that 
requested policies are always output correctly for inputs under a large number of both 
individual and simultaneous restrictions (Results E--I). These results suggest that it 
may be much more computationally difficult than is often realized to not only do LfD by
itself but also to use LfD as an initial generator of policies that are subsequently 
optimized by other techniques \cite[Sections 5.1 and 5.2]{HG+17}.
The incremental LfD results are particularly sobering, as certain
applications require that LfD be done very quickly with limited memory 
\cite[Page 5]{HG+17}, and incremental approaches that do not require
all previously-encountered demonstrations to remain available seem the best hope for
achieving this. That general solvability is not possible using efficient
probabilistic algorithms (Result C) may also be problematic, given the increasing popularity
of statistical-inference-based approaches to LfD, e.g., \cite{CKD12,CLO16}.

That being said, the various applications mentioned above may yet be guaranteed to
run efficiently by exploiting various restrictions. Our results to date suggest that it
is most important to restrict the environment. This can be done either directly by 
restricting the available features and actions in the environment ($\{|F|,|A|\}$; 
Result J) or indirectly by restricting the structure of the given demonstrations
($\{|A|, \#d, |d|, f_{eap}\}$; Result K). The latter is particularly exciting, as it
shows that LfD can be done efficiently relative to few given demonstrations as long as 
these demonstrations are also small in the sense of having few environment-state / 
action pairs which invoke few features and actions. Given that it is desirable that 
LfD be done efficiently with few demonstrations, i.e., when $\#d$ is restricted 
\cite[Page 475]{AC+09}, it would be very useful if fp-tractability held relative
to a small subset of $\{|A|, \#d, |d|, f_{eap}\}$ that includes $\#d$ (in the best case,
$\#d$ by itself). Parts (c) and (d) of Results E, F, and G rule out the 
possibility of fast LfD relative to many such subsets, including $\{\#d\}$. Just
how many (if any) of the parameters in $\{|A|, \#d, |d|, f_{eap}\}$ can be removed
while retaining fp-tractability is thus a very important open question.

This highlights the fact that any conclusions drawn from our parameterized
results must for now be tentative, as we have not yet characterized the parameterized 
complexity of all possible subsets of the parameters given in Table \ref{TabPrm}.
For example, the extent of our knowledge of the parameterized status of LfDBat 
relative to the parameters examined in this paper is brought home by the partial
intractability map in Table \ref{TabIMLfDBat}.
Tractability may lurk in the uncharacterized subsets in this map and the intractability
maps for our other LfD problems. For
example, tractability may hold relative to certain sets of restrictions on
policies (with sets including the number $c$ of allowable policy-transformation 
changes relative to problems LfDIncNoHist$^{pos}$ and LfDIncNoHist$^{neg}$ being
particularly tantalizing). Tractability might also be obtained using
parameters not considered here. Possible candidates include
parameters enforcing a high degree of similarity between time-adjacent 
environment-states in demonstrations (which seems to hold in real-world demonstrations)
or characterizing the requested degree of compression of the information in the
given demonstrations by a policy in
relative rather than absolute terms, i.e., as a ratio of given demonstration-set size
to policy size rather than (as is done here) just policy size alone. Additional
parameters may be suggested by aspects of real-world LfD instances whose values are
typically small as well as constraints invoked in cognitive systems, e.g., children 
learning by imitation \cite{BG+11,Mel05,SIB03}.

All of this is most intriguing for the basic model of LfD analyzed in this paper. 
However, what (if anything) do our results have to say about more complex models of LfD
invoked in practice? Given that real-world LfD often uses continuous rather than 
discrete demonstrations and infers a number of types of policies such as decision trees,
hidden Markov models, and Gaussian mixture models \cite[Page 789]{CKD12}, the answer 
may initially seem to be ``not much at all''. We agree that our analysis is not 
immediately applicable in these cases and are for now merely suggestive of conditions 
under which tractability and intractability may hold. However, they may be applicable 
in future in several ways:

\begin{itemize}
\item If our basic model of LfD is a special case of a more complex model of LfD, all 
       of our intractability results also apply to that more complex model (as any 
       algorithm for the more complex model would also work for our basic model, 
       making our basic model tractable and thus causing a contradiction). 
\item If this is not the case, the techniques invoked in both our tractability and 
       intractability results may suggest ways of deriving algorithms and intractability
       results for more complex models. 
\item If even this does not hold, doing classical and parameterized complexity analyses
       like those presented here may still be worthwhile, in order to characterize more 
       accurately those situations under which fast learning is and is not possible 
       relative to these more complex models. This would enable the invocation of
       LfD in more situations with greater confidence. 
\end{itemize}

\noindent
Given this potential utility of our analysis in terms of either results, proof 
techniques, or analytical frameworks, these analyses should thus be seen not 
as an endpoint but rather the start of a (hopefully ongoing  and productive) 
conversation between LfD researchers and computational complexity analysts.

\section{Conclusions}

\label{SectConc}

In this paper, we have shown how parameterized complexity analysis can be used
to systematically explore the algorithmic options for efficient and reliable
machine learning. As an illustrative example, we gave the first first parameterized
complexity analysis of batch and incremental policy inference 
under learning from demonstration (LfD). These analyses were done relative 
to a basic model of LfD which uses discrete feature-based positive and negative 
demonstrations and time-independent policies specified as single-state transducers.
Relative to this basic model, we showed that none of our LfD problems can be solved
efficiently either in general or relative to a number of (often simultaneous) 
restrictions on environments, demonstrations, and policies. We also gave the first
known restrictions under which efficient solvability is possible and discussed the
implications of our solvability and unsolvability results for both our basic model
of LfD and more complex models of LfD used in practice. 

There are several promising directions for future research, both with respect
to LfD in particular and machine learning in general. With respect to LfD,
in addition to completing the characterization of the parameterized complexity of 
the LfD problems defined in this paper relative to the parameters given in Table 
\ref{TabPrm}, analyses
should be extended to more complex models of LfD. This includes not only the more 
complex types of demonstrations and policies mentioned in Section \ref{SectDisc} 
but also more complex LfD inference problems, e.g., LfD in which learners 
interactively receive critiques from teachers \cite{ABV07} or request useful 
positive and/or negative demonstrations \cite{CV08}. More generally, given the 
popularity of statistical-inference approaches in machine
learning, it would be of great interest to extend our
parameterized exploration of algorithmic options to include probabilistic as well as
deterministic algorithms. Initial steps in this direction have already been made
for other problems \cite{BKW+13,Kwi15,MM13} and await application to problems from
machine learning. All told,
we believe that there is much that parameterized complexity analysis has 
to offer researchers in machine learning, and hope that the techniques and
analyses given in this paper are a useful first step in this endeavour.

\section*{Acknowledgments}
The author would like to thank Ting Hu and Lourdes Pe\~{n}a-Castillo
for comments on an earlier version of this paper. He would also like to acknowledge 
support for this project from National Science and Engineering Research Council 
(NSERC) Discovery Grant 228104-2015.


\appendix
\section*{Appendix A: Proofs of Results}
\label{SectProof}


In this appendix, we give proofs of various results stated in the main text
that were not given in the main text.

\begin{lemma}
{\sc Dominating set$^{PD3}_D$} polynomial-time reduces to LfDIncHist$^{pos}_D$ such that 
in the constructed instance LfDIncHist$^{pos}_D$, $|A| = \#d = 2$, $f_t = 1$, $f_{eap} =
4$, and $t$ and $c$ are functions of $k$ in the given instance of {\sc Dominating set$^{PD3}_D$}.
\label{LemRedDS_LfDIncHistPos2}
\end{lemma}

\begin{proof}
As {\sc Dominating set$^{PD3}_D$} is a special case of {\sc Dominating set}$_D$, the reduction
in Lemma \ref{LemRedDS_LfDIncHistPos1} from {\sc Dominating set}$_D$ to 
LfDIncHist$^{pos}_D$ is also a reduction from {\sc Dominating set$^{PD3}_D$} to 
LfDIncHist$^{pos}_D$ that constructs instances of LfDIncHist$^{pos}_D$ such that 
$|A| = \#d = 2$, $f_t = 1$ and $c$ and $t$ are functions of $k$ in the given instance of
{\sc Dominating set$^{PD3}_D$}. To complete the proof, note that as the degree of each 
vertex in graph $G$ in the given instance of {\sc Dominating set$^{PD3}_D$} is at most 3,
the size of each complete vertex neighbourhood is of size at most 4, which means that
$f_{eap} = 4$ in each constructed instance of LfDIncHist$^{pos}_D$.
\end{proof}

\begin{lemma}
{\sc Dominating set}$_D$ polynomial-time reduces to LfDIncHist$^{neg}_D$ such that in the 
constructed instance LfDIncHist$^{neg}_D$, $|A| = f_t = 1$, $\#d = 2$, and
$t$ and $c$ are functions of $k$ in the given instance of {\sc Dominating set}$_D$.
\label{LemRedDS_LfDIncHistNeg1}
\end{lemma}

\begin{proof}
Given an instance $\la G = (V,E), k \ra$ of {\sc Dominating Set}$_D$, construct an instance
$\la D, p, d_{new}, c, t, f_t \ra$ of LfDIncHist$^{neg}_D$ as follows: Let $F = \{f_1,
f_2, \ldots,$ $f_{|V|}, f_x\}$, $A = \{a\}$, and $D$ be as in the reduction in the proof
 of Lemma \ref{LemRedDS_LfDIncHistPos1}. Let $p$ have $t = k$ transitions, where the
first transition is $(\{f_x\}, a)$ and the remaining $k - 1$ transitions have the form
$(\{f_i\}, a)$ where $f_i$ is the feature corresponding to a randomly selected
vertex in $G$. Finally, let $d_{new} = (neg, \la (\{f_x\}, a)\ra)$, $c = k$ and $f_t =
1$. Note that $p$ is valid for $D$ (as all transitions produce the same action) and 
consistent with $D$ (as the first transition in $T$ will always generate the correct 
action $a$ for each demonstration in $D$). Observe that this construction can be done
in time polynomial in the size of the given instance of {\sc Dominating set}$_D$.

We shall prove the correctness of this reduction in two parts. First, suppose that there
is a subset $V' = \{v'_1, v'_2, \ldots v'_l\} \subseteq V$, $l \leq k$, that is a 
dominating set in $G$. Construct a policy $p'$ with $l$ transitions which 
have the form $(\{f_i\}, a)$ for each $v'_i \in V'$. Observe that
$p'$ can be derived from $p$ by at most $c = k$ changes to $p$ (namely, change the 
triggering feature-sets of the first $l$ transitions as necessary and delete the final
$k - l$ transitions) and that $p'$ is valid for
$D \cup \{d_{new}\}$ (as all transitions produce the same action). As $V'$ is a 
dominating set in $G$ and the state in each demonstration in $D$ corresponds to the 
complete neighbourhood of one of the vertices in $G$, $p'$ will produce the correct 
action for every demonstration in $D$ and hence is consistent with $D$. Moreover, 
$p'$ cannot produce the action forbidden by $d_{new}$ for state $\{f_x\}$,
which means that $p'$ is consistent with $D \cup \{d_{new}\}$.

Conversely, suppose there is a policy $p'$ derivable from $p$ by at most $c$ changes
that is valid for and consistent with $D \cup \{d_{new}\}$ and has $l \leq t = k$ 
transitions, each of which is triggered by a set of at most $f_t$ features. None of
these transitions can produce action $a$ on state $\{f_x\}$ in order for $p'$ to be 
consistent with $d_{new}$; moreover, the $l$ transitions in $p'$ must produce
action $a$ for all states in $D$ in order for $p'$ to be consistent with 
$D$. As $f_t = 1$ and the states in the demonstrations in $D$ correspond to the complete
neighborhoods of the vertices in $G$, the set of features triggering these $l$ 
transitions in $p$ must correspond to a dominating set of size at most $k$ in $G$. 

To complete the proof, observe that in the constructed instance of 
LfDIncHist$^{neg}_D$, $|A| = f_t = 1$, $\#d = 2$, and $c = t = k$.
\end{proof}

\begin{lemma}
{\sc Dominating set$^{PD3}_D$} polynomial-time reduces to LfDIncHist$^{neg}_D$ such that 
in the constructed instance LfDIncHist$^{neg}_D$, $|A| = f_t = 1$, $\#d = 2$, 
$f_{eap} = 4$, and $t$ and $c$ are functions of $k$ in the given instance of 
{\sc Dominating set$^{PD3}_D$}.
\label{LemRedDS_LfDIncHistNeg2}
\end{lemma}

\begin{proof}
As {\sc Dominating set$^{PD3}_D$} is a special case of {\sc Dominating set}$_D$, the reduction
in Lemma \ref{LemRedDS_LfDIncHistNeg1} from {\sc Dominating set}$_D$ to 
LfDIncHist$^{neg}_D$ is also a reduction from {\sc Dominating set$^{PD3}_D$} to 
LfDIncHist$^{neg}_D$ that constructs instances of LfDIncHist$^{neg}_D$ such that $|A| =
f_t = 1$, $\#d = 2$, and $c$ and $t$ are functions of $k$ in the given instance of
{\sc Dominating set$^{PD3}_D$}. To complete the proof, note that as the degree of each 
vertex in graph $G$ in the given instance of {\sc Dominating set$^{PD3}_D$} is at most 3,
the size of each complete vertex neighbourhood is of size at most 4, which means that
$f_{eap} = 4$ in each constructed instance of LfDIncHist$^{neg}_D$.
\end{proof}

\begin{lemma}
{\sc Dominating set}$_D$ polynomial-time reduces to LfDIncNoHist$^{pos}_D$ such that in the 
constructed instance LfDIncNoHist$^{pos}_D$, $|A| = \#d = f_t = 1$, and
$t$ is a function of $k$ in the given instance of {\sc Dominating set}$_D$.
\label{LemRedDS_LfDIncNoHistPos1}
\end{lemma}

\begin{proof}
Given an instance $\la G = (V,E), k \ra$ of {\sc Dominating Set}$_D$, construct an instance
$\la p, d_{new}, c, t, f_t \ra$ of LfDIncNoHist$^{pos}_D$ as follows: Let $F = \{f_1,
f_2, \ldots, f_{|V|},$ $f_x\}$, $A = \{a\}$, $p$ have $|V| + 1$ transitions 
such that the $i$th, $1 \leq i \leq |V|$, transition has a triggering feature-set 
consisting of the features in $F$ corresponding to the complete neighbourhood of $v_i$ 
in $G$ and action $a$ and the final transition has a triggering feature-set that is an 
arbitrary subset of the features of $F$ that is not the same as the triggering 
feature-set of any previous transition and action $a$, $d_{new} = (pos, \la (\{f_x\}, 
a)\ra)$, $c = |V| + 1$, $t = k + 1$, and $f_t = 1$.  Observe that this 
construction can be done in time polynomial in the size of the given instance of 
{\sc Dominating set}$_D$.

We shall prove the correctness of this reduction in two parts. First, suppose that there
is a subset $V' = \{v'_1, v'_2, \ldots v'_l\} \subseteq V$, $l \leq k$, that is a 
dominating set in $G$. Construct a policy $p'$ with $l + 1$ transitions in which
the first transition is $(\{f_x\}, a)$ and the subsequent $l$ transitions have the
form $(\{f_i\}, a)$ for each $v'_i \in V'$. Observe that $p'$ can be derived from $p$
by at most $c$ changes to $p$ (namely, change the feature-sets
of the first $l + 1$ transitions as necessary
and delete the final $(|V| + 1) - (l + 1)$ transitions).
As $V'$ is a dominating set in $G$ and the triggering feature-set in each
transition in $p$ corresponds to the complete neighbourhood of one of the vertices in
$G$, $p'$ will produce the correct action for the triggering feature-set associated
with each transition in $p$; moreover, the first transition in $p'$ produces the correct 
action for $d_{new}$, which means that $p'$ is consistent with $p$ modulo $d_{new}$.

Conversely, suppose there is a policy $p'$ derivable from $p$ by at most $c$ changes 
that is consistent with $p$ modulo $d_{new}$ and has $l \leq t = k + 1$ transitions, 
each of which is triggered by a set of at most $f_t$ features. One of these 
transitions must produce action $a$ relative to $f_x$ in order for $p'$ to be 
consistent with $d_{new}$; moreover, as $f_x$ does not occur in the triggering 
feature-set of any transition in $p$, the remaining $l - 1$ transitions in $p'$ must
produce action $a$ for all of these triggering feature-sets in order for $p'$ to be
consistent with $p$. As $f_t = 1$ and the triggering feature-sets in the transitions in $p$ 
correspond to the complete neighborhoods of the vertices in $G$, the set of features 
triggering the final $l - 1$ transitions in $p'$ must correspond to a dominating set 
of size at most $k$ in $G$. 

To complete the proof, observe that in the constructed instance of  
LfDIncNoHist$^{pos}_D$, $|A| = \#d = f_t = 1$ and $t = k + 1$.
\end{proof}

\begin{lemma}
{\sc Dominating set$^{PD3}_D$} polynomial-time reduces to LfDIncNoHist$^{pos}_D$ such that
in the constructed instance LfDIncNoHist$^{pos}_D$, $|A| = \#d = f_t = 1$, $f_{eap} = 
4$, and $t$ is a function of $k$ in the given instance of {\sc Dominating set$^{PD3}_D$}.
\label{LemRedDS_LfDIncNoHistPos2}
\end{lemma}

\begin{proof}
As {\sc Dominating set$^{PD3}_D$} is a special case of {\sc Dominating set}$_D$, the reduction
in Lemma \ref{LemRedDS_LfDIncNoHistPos1} from {\sc Dominating set}$_D$ to 
LfDIncNoHist$^{pos}_D$ is also a reduction from {\sc Dominating set$^{PD3}_D$} to 
LfDIncNoHist$^{pos}_D$ that constructs instances of LfDIncNoHist$^{pos}_D$ such that 
$|A| = \#d = f_t = 1$ and $t$ is a function of $k$ in the given instance of 
{\sc Dominating set$^{PD3}_D$}. To complete the proof, note that as the degree of each 
vertex in graph $G$ in the given instance of {\sc Dominating set$^{PD3}_D$} is at most 3,
the size of each complete vertex neighbourhood is of size at most 4, which means that
$f_{eap} = 4$ in each constructed instance of LfDIncNoHist$^{pos}_D$.
\end{proof}

\begin{lemma}
{\sc Dominating set}$_D$ polynomial-time reduces to LfDIncNoHist$^{neg}_D$ such that in the 
constructed instance LfDIncNoHist$^{neg}_D$, $|A| = \#d = f_t = 1$ and
$t$ is a function of $k$ in the given instance of {\sc Dominating set}$_D$.
\label{LemRedDS_LfDIncNoHistNeg1}
\end{lemma}

\begin{proof}
Given an instance $\la G = (V,E), k \ra$ of {\sc Dominating Set}$_D$, construct an instance
$\la p, d_{new}, c, t, f_t \ra$ of LfDIncNoHist$^{neg}_D$ as follows: Let $F = \{f_1,
f_2, \ldots, f_{|V|},$ $f_x\}$, $A = \{a\}$, $p$ have $|V|$ transitions such that the
$i$th, $1 \leq i \leq |V|$, transition has a triggering feature-set consisting
of the features in $F$ corresponding to the complete neighbourhood of $v_i$ in $G$
and action $a$, $d_{new} = (neg, \la (\{f_x\}, a)\ra)$, $c = |V|$, $t = k$, and 
$f_t = 1$.  Observe that this construction can be done in time polynomial in the 
size of the given instance of {\sc Dominating set}$_D$.

We shall prove the correctness of this reduction in two parts. First, suppose that there
is a subset $V' = \{v'_1, v'_2, \ldots v'_l\} \subseteq V$, $l \leq k$, that is a 
dominating set in $G$. Construct a policy $p'$ with $l$ transitions in which
the $i$th, $1 \leq i \leq l$, transition has the form $(\{f_i\}, a)$ for each 
$v'_i \in V'$. Observe that $p'$ can be derived from $p$ by at most $c$ changes to 
$p$ (namely, change feature-sets of the first $l$ transitions as necessary
and delete the final $|V| - l$ transitions). As $V'$ is a dominating set in $G$ and the 
triggering feature-set in each
transition in $p$ corresponds to the complete neighbourhood of one of the vertices in
$G$, $p'$ will produce the correct action for the triggering feature-set associated
with each transition in $p$; moreover, $p'$ cannot produce 
action $a$ for $d_{new}$, which means that $p'$ is consistent with $p$ modulo $d_{new}$.

Conversely, suppose there is a policy $p'$ derivable from $p$ by at most $c$ changes 
that is consistent with $p$ modulo $d_{new}$ and has $l \leq t = k$ transitions, 
each of which is triggered by a set of at most $f_t$ features. The $l$ transitions in $p'$ must
produce action $a$ for all of the triggering feature-sets in $p$ in order for $p'$ to be
consistent with $p$. As $f_t = 1$ and the triggering feature-sets in the transitions in $p$ 
correspond to the complete neighborhoods of the vertices in $G$, the set of features 
triggering the $l$ transitions in $p'$ must correspond to a dominating set 
of size at most $k$ in $G$. 

To complete the proof, observe that in the constructed instance of  
LfDIncNoHist$^{pos}_D$, $|A| = \#d = f_t = 1$ and $t = k$.
\end{proof}

\begin{lemma}
{\sc Dominating set$^{PD3}_D$} polynomial-time reduces to LfDIncNoHist$^{neg}_D$ such that
in the constructed instance LfDIncNoHist$^{pos}_D$, $|A| = \#d = f_t = 1$, $f_{eap} = 
4$, and $t$ is a function of $k$ in the given instance of {\sc Dominating set$^{PD3}_D$}.
\label{LemRedDS_LfDIncNoHistNeg2}
\end{lemma}

\begin{proof}
As {\sc Dominating set$^{PD3}_D$} is a special case of {\sc Dominating set}$_D$, the reduction
in Lemma \ref{LemRedDS_LfDIncNoHistNeg1} from {\sc Dominating set}$_D$ to 
LfDIncNoHist$^{neg}_D$ is also a reduction from {\sc Dominating set$^{PD3}_D$} to 
LfDIncNoHist$^{neg}_D$ that constructs instances of LfDIncNoHist$^{neg}_D$ such that 
$|A| = \#d = f_t = 1$ and $t$ is a function of $k$ in the given instance of 
{\sc Dominating set$^{PD3}_D$}. To complete the proof, note that as the degree of each 
vertex in graph $G$ in the given instance of {\sc Dominating set$^{PD3}_D$} is at most 3,
the size of each complete vertex neighbourhood is of size at most 4, which means that
$f_{eap} = 4$ in each constructed instance of LfDIncNoHist$^{neg}_D$.
\end{proof}

\vspace*{0.15in}

\noindent
{\bf Result F}: LfDIncHist$^{pos}$ is not fp-tractable relative to the following 
                 parameter-sets:

                \begin{quote}
                \begin{description}
                \item[a)] $\{ |A|, \#d, t, f_t, c\}$
                            when $|A| = \#d = 2$ and $f_t = 1$ (unless $FPT = W[1]$)
                \item[b)] $\{ |A|, |d|, t, f_t, c\}$
                            when $|A| = |d| = 2$ and $f_t = 1$ (unless $FPT = W[1]$)
                \item[c)] $\{ |A|, \#d, f_{eap}, f_t\}$
                            when $|A| = \#d = 2$, $f_{eap} = 4$, and $f_t = 1$ 
                            (unless $P = NP$)
                \item[d)] $\{ |A|, |d|, f_{eap}, f_t\}$
                            when $|A| = |d| = 2$, $f_{eap} = 4$, and $f_t = 1$ 
                            (unless $P = NP$)
                \end{description}
                \end{quote}

\begin{proof} \hspace*{0.5in} \\
\noindent
{\em Proof of part (a)}: Follows from the $W[2]$-hardness of 
$\la k \ra$-{\sc Dominating set}$_D$, the reduction in Lemma \ref{LemRedDS_LfDIncHistPos1},
the inclusion of $W[1]$ in $W[2]$, and the conjecture $FPT \neq W[1]$.

\noindent
{\em Proof of part (b)}: Follows from part (a) and Lemma \ref{LemPropDemo}.

\noindent
{\em Proof of part (c)}: Follows from the $NP$-hardness of {\sc Dominating set$^{PD3}_D$},
the reduction in Lemma \ref{LemRedDS_LfDIncHistPos2}, and Lemma \ref{LemProp2}.

\noindent
{\em Proof of part (d)}: Follows from part (c) and Lemma \ref{LemPropDemo}.
\end{proof}

\vspace*{0.15in}

\noindent
{\bf Result G}: LfDIncHist$^{neg}$ is not fp-tractable relative to the following 
                 parameter-sets:

                \begin{quote}
                \begin{description}
                \item[a)] $\{ |A|, \#d, t, f_t, c\}$
                            when $|A| = \#d = 2$ and $f_t = 1$ (unless $FPT = W[1]$)
                \item[b)] $\{ |A|, |d|, t, f_t, c\}$
                            when $|A| = |d| = 2$ and $f_t = 1$ (unless $FPT = W[1]$)
                \item[c)] $\{ |A|, \#d, f_{eap}, f_t\}$
                            when $|A| = \#d = 2$, $f_{eap} = 4$, and $f_t = 1$ 
                            (unless $P = NP$)
                \item[d)] $\{ |A|, |d|, f_{eap}, f_t\}$
                            when $|A| = |d| = 2$, $f_{eap} = 4$, and $f_t = 1$ 
                            (unless $P = NP$)
                \end{description}
                \end{quote}

\begin{proof} \hspace*{0.5in} \\
\noindent
{\em Proof of part (a)}: Follows from the $W[2]$-hardness of 
$\la k \ra$-{\sc Dominating set}$_D$, the reduction in Lemma \ref{LemRedDS_LfDIncHistNeg1},
the inclusion of $W[1]$ in $W[2]$, and the conjecture $FPT \neq W[1]$.

\noindent
{\em Proof of part (b)}: Follows from part (a) and Lemma \ref{LemPropDemo}.

\noindent
{\em Proof of part (c)}: Follows from the $NP$-hardness of {\sc Dominating set$^{PD3}_D$},
the reduction in Lemma \ref{LemRedDS_LfDIncHistNeg2}, and Lemma \ref{LemProp2}.

\noindent
{\em Proof of part (d)}: Follows from part (c) and Lemma \ref{LemPropDemo}.
\end{proof}

\vspace*{0.15in}

\noindent
{\bf Result H}: LfDIncNoHist$^{pos}$ is not fp-tractable relative to the following 
                 parameter-sets:

                \begin{quote}
                \begin{description}
                \item[a)] $\{ |A|, \#d, t, f_t\}$
                            when $|A| = \#d = f_t = 1$ (unless $FPT = W[1]$)
                \item[b)] $\{ |A|, |d|, t, f_t\}$
                            when $|A| = |d| = f_t = 1$ (unless $FPT = W[1]$)
                \item[c)] $\{ |A|, \#d, f_{eap}, f_t\}$
                            when $|A| = \#d = f_t = 1$ and $f_{eap} = 4$ 
                             (unless $P = NP$)
                \item[d)] $\{ |A|, |d|, f_{eap}, f_t\}$
                            when $|A| = |d| = f_t = 1$ and $f_{eap} = 4$ 
                             (unless $P = NP$)
                \end{description}
                \end{quote}

\begin{proof} \hspace*{0.5in} \\
\noindent
{\em Proof of part (a)}: Follows from the $W[2]$-hardness of 
$\la k \ra$-{\sc Dominating set}$_D$, the reduction in Lemma \ref{LemRedDS_LfDIncNoHistPos1},
the inclusion of $W[1]$ in $W[2]$, and the conjecture $FPT \neq W[1]$.

\noindent
{\em Proof of part (b)}: Follows from part (a) and Lemma \ref{LemPropDemo}.

\noindent
{\em Proof of part (c)}: Follows from the $NP$-hardness of {\sc Dominating set$^{PD3}_D$},
the reduction in Lemma \ref{LemRedDS_LfDIncNoHistPos2}, and Lemma \ref{LemProp2}.

\noindent
{\em Proof of part (d)}: Follows from part (c) and Lemma \ref{LemPropDemo}.
\end{proof}

\vspace*{0.15in}

\noindent
{\bf Result I}: LfDIncNoHist$^{neg}$ is not fp-tractable relative to the following 
                 parameter-sets:

                \begin{quote}
                \begin{description}
                \item[a)] $\{ |A|, \#d, t, f_t\}$
                            when $|A| = \#d = f_t = 1$ (unless $FPT = W[1]$)
                \item[b)] $\{ |A|, |d|, t, f_t\}$
                            when $|A| = |d| = f_t = 1$ (unless $FPT = W[1]$)
                \item[c)] $\{ |A|, \#d, f_{eap}, f_t\}$
                            when $|A| = \#d = f_t = 1$ and $f_{eap} = 4$ 
                             (unless $P = NP$)
                \item[d)] $\{ |A|, |d|, f_{eap}, f_t\}$
                            when $|A| = |d| = f_t = 1$ and $f_{eap} = 4$ 
                             (unless $P = NP$)
                \end{description}
                \end{quote}

\begin{proof} \hspace*{0.5in} \\
\noindent
{\em Proof of part (a)}: Follows from the $W[2]$-hardness of 
$\la k \ra$-{\sc Dominating set}$_D$, the reduction in Lemma \ref{LemRedDS_LfDIncNoHistNeg1},
the inclusion of $W[1]$ in $W[2]$, and the conjecture $FPT \neq W[1]$.

\noindent
{\em Proof of part (b)}: Follows from part (a) and Lemma \ref{LemPropDemo}.

\noindent
{\em Proof of part (c)}: Follows from the $NP$-hardness of {\sc Dominating set$^{PD3}_D$},
the reduction in Lemma \ref{LemRedDS_LfDIncNoHistNeg2}, and Lemma \ref{LemProp2}.

\noindent
{\em Proof of part (d)}: Follows from part (c) and Lemma \ref{LemPropDemo}.
\end{proof}


\end{document}